\documentclass{article}

\usepackage{microtype}
\usepackage{graphicx}
\usepackage{subcaption}
\usepackage{booktabs} % for professional tables
\usepackage{xcolor}%！

\usepackage{hyperref}

\usepackage[accepted]{icml2026}

\usepackage{amsmath}
\usepackage{amssymb}
\usepackage{mathtools}
\usepackage{amsthm}
\usepackage{multirow}

\usepackage[capitalize,noabbrev]{cleveref}

\theoremstyle{plain}
\newtheorem{theorem}{Theorem}[section]

\newtheorem{lemma}[theorem]{Lemma}
\newtheorem{corollary}[theorem]{Corollary}
\theoremstyle{definition}
\newtheorem{definition}[theorem]{Definition}
\newtheorem{assumption}[theorem]{Assumption}
\theoremstyle{remark}
\newtheorem{remark}[theorem]{Remark}

\usepackage[textsize=tiny]{todonotes}

\icmltitlerunning{Sharper Generalization Bounds for Transformer}

\begin{document}

\twocolumn[
\icmltitle{Sharper Generalization Bounds for Transformer}



\icmlsetsymbol{equal}{*}
\icmlsetsymbol{corr}{$^\dagger$}

\begin{icmlauthorlist}
\icmlauthor{Yawen Li}{equal,aaa}
\icmlauthor{Tao Hu}{equal,aaa}
\icmlauthor{Zhouhui Lian}{ccc}
\icmlauthor{Wan Tian}{corr,ccc,bbb}
\icmlauthor{Yijie Peng}{corr,ddd,eee,fff}
\icmlauthor{Huiming Zhang}{ggg}
\icmlauthor{Zhongyi Li}{ggg}
\end{icmlauthorlist}

\icmlaffiliation{aaa}{School of Mathematical Sciences, Capital Normal University, 100048 Beijing, China}
\icmlaffiliation{bbb}{Advanced Institute of Information Technology, Peking University}
\icmlaffiliation{ccc}{Wangxuan Institute of Computer Technology, Peking University, China, 100871}
\icmlaffiliation{ddd}{PKU-Wuhan Institute for Artificial Intelligence}
\icmlaffiliation{eee}{Xiangjiang Laboratory, Changsha 410000, China}
\icmlaffiliation{fff}{Guanghua School of Management, Peking University, Beijing,  China, 100871}
\icmlaffiliation{ggg}{School of Artificial Intelligence, Beihang University, 100191 Beijing, China}

\icmlcorrespondingauthor{Wan Tian}{wantian61@foxmail.com}
\icmlcorrespondingauthor{Yijie Peng}{pengyijie@gsm.pku.edu.cn}

\icmlkeywords{Machine Learning, ICML}

\vskip 0.3in
]



\printAffiliationsAndNotice{}  % no special notice (required even if empty)

\begin{abstract}
This paper studies generalization error bounds for Transformer models. Based on the offset Rademacher complexity, we derive sharper generalization bounds for different Transformer architectures, including single-layer single-head, single-layer multi-head, and multi-layer Transformers. We first express the excess risk of Transformers in terms of the offset Rademacher complexity. By exploiting its connection with the empirical covering numbers of the corresponding hypothesis spaces, we obtain excess risk bounds that achieve optimal convergence rates up to constant factors. We then derive refined excess risk bounds by upper bounding the covering numbers of Transformer hypothesis spaces using matrix ranks and matrix norms, leading to precise, architecture-dependent generalization bounds. Finally, we relax the boundedness assumption on feature mappings and extend our theoretical results to settings with unbounded (sub-Gaussian) features and heavy-tailed distributions.
\end{abstract}

\section{Introduction}
Transformer-based models have become a central component of modern machine learning systems and have achieved remarkable success across a wide range of application domains \citep{chang2024survey, de2022attention}. Originally developed for sequence modeling tasks \citep{vaswani2017attention}, Transformers now underpin state-of-the-art performance in natural language processing \citep {wang2019learning,zhang2023survey}, computer vision \citep {han2022survey}, and reinforcement learning \citep {chen2021decision, hu2024transforming} . Their architectural flexibility and strong empirical performance have also led to widespread adoption in large-scale foundation models.

Despite their empirical success, understanding the generalization behavior of Transformer models remains a fundamental theoretical challenge. Providing sharp and architecture-aware generalization error bounds is essential for explaining their empirical robustness and for developing principled insights into the role of architectural design in learning performance. Existing nonparametric and functional-analytic approaches remain limited in explaining the generalization behavior of deep models: nonparametric theories typically rely on Hölder-type smoothness assumptions on the target function \citep{schmidt2020nonparametric, bos2022convergence} , yielding error bounds that characterize worst-case function classes and thus fail to capture the adaptivity of deep networks in high-dimensional structured settings; analyses based on Sobolev and related function spaces primarily focus on approximation properties and are often derived under idealized or infinite-sample assumptions \citep{yang2024optimal, Ding2025Semi, meng2022new, ma2022barron}, providing insufficient characterization of finite-sample statistical errors and the influence of training algorithms.

In this work, we analyze the generalization ability of Transformer models from a complexity-based perspective. Compared with analyses relying on smoothness or approximation assumptions, complexity-based approaches yield generalization bounds that explicitly depend on the sample size, model scale, and parameter norms, making them more suitable for modern deep models such as Transformers with large parameterization and highly complex architectures. The core idea of this approach is to characterize the excess risk by means of complexity measures of the hypothesis space, and then derive upper bounds on these measures, typically based on structural properties of the hypothesis class, such as its covering number or Vapnik–Chervonenkis (VC) dimension \citep{blumer1989learnability}. The \emph{excess risk} evaluate the performance of the estimator $\widehat{f}_n$:
\begin{equation}
\label{e:excess-risk}
\mathcal{E}(\widehat{f}_n; \ell) \;:=\; R(\widehat{f}_n) \;-\; \inf_{g \in \mathcal{J}} R(g),
\end{equation}
where $\mathcal{J}$ is a target function class. If $\mathcal{F} \subsetneq \mathcal{J}$ (the misspecified setting), this term decomposes into estimation error and approximation error.

%Commonly used complexity measures include Gaussian complexity, global Rademacher complexity, and local Rademacher complexity \citep{zhang2023mathematical}. 

Different complexity measures yield different convergence rates for the excess risk. Using Gaussian complexity or global Rademacher complexity typically leads to a convergence rate of $\mathcal{O}(1/\sqrt{n})$ \citep{zhang2023mathematical}, where $n$ denotes the sample size. In contrast to global Rademacher complexity, local Rademacher complexity is a data-dependent complexity measure that focuses on “local” subsets of the hypothesis space—particularly those functions with small empirical risk—thus avoiding overly pessimistic penalties on the entire function class \citep{bartlett2005local}. This enables sharper generalization bounds with a faster convergence rate of $\mathcal{O}(1/n)$. However, this approach relies on an assumption about the noise level of the loss function—specifically, that the variance of the loss is upper-bounded by its expectation—known as the Bernstein condition—which may be difficult to verify in real-world applications. More recently, offset Rademacher complexity, introduced by Liang et al. \citep{liang2015learning}, denoted by
$$\mathcal{R}_n^{\text{off}}(\mathcal{F},\beta)
= \mathbb{E}\left[
\sup_{f\in\mathcal{F}}
\frac{1}{n}\sum_{i=1}^n \tau_i f(X_i) - \beta f(X_i)^2
\right],$$
has been proposed as a penalized variant of global Rademacher complexity. It achieves the optimal convergence rate without explicitly imposing the Bernstein condition. This measure has been successfully applied to a wide range of models—including parametric models, nonparametric models, and neural networks—to improve convergence rates \citep{duan2023fast}. Given these theoretical advancements, several studies have already explored the generalization error of Transformers from different perspectives. Our contributions can be summarized as follows:
\begin{itemize}
\item We provide optimal convergence rates for the excess risk of various Transformer architectures, taking their structural parameters into account. Specifically, we first derive the relationship between the excess risk and offset Rademacher complexity for single-layer single-head, single-layer multi-head, and multi-layer Transformers, and then establish the connection between offset Rademacher complexity and the empirical covering numbers of the corresponding hypothesis spaces, which yields the optimal convergence rate of $\mathcal{O}(1/n)$.
\item We analyze the covering number upper bounds of Transformers from two perspectives—rank and norm—thereby obtaining precise generalization error bounds for Transformer models. 
\item We extend our theoretical findings that require certain boundedness assumptions (e.g., on feature maps), to unbounded (sub-Gaussian) and heavy-tailed settings. Finally, we demonstrate the applicability of our theoretical results to regression and classification tasks, as well as to scenarios involving robust loss functions.
\end{itemize}

\section{Related Work}  
% Regarding the generalization error theory of Transformer,  \citet{trauger2023sequence} use various covering number bounds for linear function classes to derive generalization error bounds for single-layer Transformers, achieving a decay rate of $\mathcal{O}(1/\sqrt{n})$ and highlighting the benefits of low-rank matrices in Transformer design. Similarly, \citet{truong2024rank} develop norm-based generalization bounds for Transformers using global Rademacher complexity, also independent of sequence length. Additionally, some studies investigate the generalization capabilities of Transformers from the perspective of the benign overfitting phenomenon. For example, \citet{havrilla2024understanding} demonstrate that when data lie on low-dimensional manifolds, the generalization error of Transformers exhibits a power-law scaling behavior. This finding offers a novel theoretical perspective on the generalization properties of Transformers. However, these theoretical results only yield suboptimal convergence rates for Transformers.

Regarding generalization theory for Transformers, \citet{trauger2023sequence} derive sequence-length--independent bounds for single-layer Transformers via covering-number control of associated linear classes, highlighting how low-rank structure can tighten complexity and yielding $\mathcal{O}(1/\sqrt{n})$ rates. Similarly, \citet{truong2024rank} develop norm- and rank-dependent bounds using (global) Rademacher complexity, also independent of sequence length. Beyond capacity-based analyses, \citet{havrilla2024understanding} explain empirical scaling laws through a benign-overfitting lens when data concentrate on low-dimensional manifolds, while \citet{zhang2025understandinggeneralizationtransformerserror} further connect generalization error to training dynamics by distinguishing benign versus harmful overfitting under label-flip noise. 

Recent work has also expanded toward regime- and task-structured guarantees: \citet{mwigo2026generalization} establish norm/Rademacher bounds for shallow Transformers trained by gradient descent in the lazy-training regime; \citet{huang2025formalframeworkunderstandinglength} provide a formal framework for length generalization in causal Transformers with learnable absolute positional encodings; \citet{alokhina2026smalllargegeneralizationbounds} characterize size generalization on variable-size geometric inputs via discrete-to-continuous approximation under stable positional encodings; and \citet{zhang2026generalizationboundstransformerchannel} derive bit-wise Rademacher-complexity bounds for Transformer channel decoders, leveraging sparsity from parity-check masked attention to tighten covering arguments. Despite this progress, many existing results rely on \emph{global} complexity measures and consequently yield conservative (often suboptimal) excess-risk rates, and they typically impose boundedness assumptions on feature mappings; in contrast, our work extends the theory to unbounded (sub-Gaussian) features and further to heavy-tailed distributions.

\section{Preliminaries} \label{sec:preliminaries}

In this section, we establish the mathematical framework used throughout the paper. We begin by defining the Transformer architecture—spanning single-head, multi-head, and multi-layer variants—and then formalize the statistical learning setup, including the loss functions, risk definitions, and the complexity measures that underpin our generalization analysis.

\subsection{Self-Attention and Transformers} \label{subsec:transformers}

Unlike standard feedforward networks with fixed connectivity, Transformers utilize a self-attention mechanism where weights are data-dependent and recomputed for every input. This allows each position in a sequence (e.g., a token in text) to attend to all other positions, preserving long-range dependencies.
% \begin{figure}[ht]
%   \vskip 0.2in
%   \begin{center}
%     \centerline{\includegraphics[width=\columnwidth]{fig/fig2.png}}
%     \caption{
%       Illustration of the self-attention mechanism and the multi-layer Transformer architecture.
%     }
%     \label{fig:transformer_arch}
%   \end{center}
% \end{figure}

Let $X \in \mathbb{R}^{T \times d}$ denote the input sequence with length $T$ and embedding dimension $d$. We define the row-wise softmax operator, $\operatorname{softmax}: \mathbb{R}^{T \times T} \to \mathbb{R}^{T \times T}$, such that for any matrix $A$, the entry $(i,j)$ of $\operatorname{softmax}(A)$ is given by $\exp(A_{ij}) / \sum_{k=1}^T \exp(A_{ik})$. Let $\sigma: \mathbb{R} \to \mathbb{R}$ be an element-wise $L_\sigma$-Lipschitz activation function with $\sigma(0)=0$.

\paragraph{Single-Head Attention.}
A single attention head is parameterized by matrices $W_Q, W_K \in \mathbb{R}^{d \times d}$, $W_v \in \mathbb{R}^{d \times k}$, and $W_c \in \mathbb{R}^{k \times d}$. For notational convenience, we denote the query-key interaction matrix as $W_{QK} := W_Q W_K^\top$. The output of a single-head layer is defined as:
\begin{equation} \label{eq:single-head}
\begin{aligned}
f_{\mathrm{SH}}(X)&=\sigma\left( \operatorname{softmax}\left( X W_{QK} X^\top \right) X W_v \right) W_c \\
&\in\mathbb{R}^{T \times d}.
\end{aligned}
\end{equation}
To obtain a scalar prediction for regression or classification, we assume the input sequence contains a special token (e.g., \texttt{[CLS]}). Let $Y_{\texttt{[CLS]}} \in \mathbb{R}^d$ be the row corresponding to this token in the output of \eqref{eq:single-head}. The final prediction is given by $w^\top Y_{\texttt{[CLS]}}$, where $w \in \mathbb{R}^d$ is a learnable readout vector.

\paragraph{Multi-Head Attention.}
In a multi-head Transformer with $H$ heads, the model aggregates the outputs of independent heads. Let the $h$-th head be parameterized by $\{W_{h,Q}, W_{h,K}, W_{h,v}, W_{h,c}\}$. The layer output is given by:
\begin{equation} \label{eq:multi-head}
\begin{aligned}
&f_{\mathrm{MH}}(X) \\
& =\sum_{h=1}^{H} \sigma\left( \operatorname{softmax}\left( X W_{h,QK} X^\top \right) X W_{h,v} \right) W_{h,c}.
\end{aligned}
\end{equation}
The scalar prediction is obtained via the readout vector $w$ applied to the aggregated \texttt{[CLS]} representation: $w^\top (\sum_{h=1}^H (Y_h)_{\texttt{[CLS]}})$.

\paragraph{Multi-Layer Architecture.}
Deep Transformers are constructed by stacking layers, often incorporating normalization to aid optimization. Let $X^{(l)}$ denote the input to the $l$-th layer, with $X^{(1)} = X$. The $l$-th block is parameterized by $\mathcal{W}^{(l)} = \{W_{QK}^{(l)}, W_{v}^{(l)}, W_{c}^{(l)}\}$. We define the intermediate attention mapping as:
\begin{equation} \label{eq:layer-attention}
\begin{aligned}
&\Phi(X^{(l)}; \mathcal{W}^{(l)}) \\
&:=\sigma\left( \operatorname{softmax}\left( X^{(l)} W_{QK}^{(l)} (X^{(l)})^\top \right) X^{(l)} W_{v}^{(l)} \right) W_{c}^{(l)}.
\end{aligned}
\end{equation}
Let $\Pi_{\mathrm{norm}}$ denote a row-wise normalization operator (e.g., projection onto the unit $\ell_2$-ball). The recursive update for the $(l+1)$-th layer input is defined as:
\begin{equation} \label{eq:layered}
X^{(l+1)} \;=\; \Pi_{\mathrm{norm}}\left( \sigma\left( \Pi_{\mathrm{norm}}\left( \Phi(X^{(l)}; \mathcal{W}^{(l)}) \right) \right) \right).
\end{equation}
This formulation encapsulates the dimension-preserving nature of the Transformer block while explicitly accounting for normalization and activation steps crucial for theoretical stability.

\begin{table}[t]
\caption{Classification accuracies for naive Bayes and flexible
Bayes on various data sets.}
\label{sample-table}
\begin{center}
\begin{small}
\begin{sc}
\begin{tabular}{lcccr}
\toprule
\textbf{Model } & \textbf{Definition} & \textbf{Parameters} \\
\midrule
\multirow{2}{*}{Single-Head} & \multirow{2}{*}{Eq. \eqref{eq:single-head}} & $W_{QK} \in \mathbb{R}^{d \times d}, $\\
& &
$W_v \in \mathbb{R}^{d \times k}, W_c \in \mathbb{R}^{k \times d}$\\
Multi-Head & Eq. \eqref{eq:multi-head} & $\{W_{h,QK}, W_{h,v}, W_{h,c}\}_{h=1}^H$ \\
Multi-Layer & Eq. \eqref{eq:layered} & $\{W_{QK}^{(l)}, W_{v}^{(l)}, W_{c}^{(l)}\}_{l=1}^L$ \\
\bottomrule
\end{tabular}
\end{sc}
\end{small}
\end{center}
\vskip -0.1in
\end{table}

\subsection{Excess Risk and Complexity Measures} \label{subsec:excess_risk}

We consider the supervised learning setting with input-output pairs $Z = (X, Y) \in \mathcal{Z} := \mathcal{X} \times \mathcal{Y}$, where $\mathcal{X} \subseteq \mathbb{R}^{T \times d}$ and $\mathcal{Y} \subseteq \mathbb{R}$. We are given an i.i.d.\ sample $\mathbb{D} = \{Z_i\}_{i=1}^n$ drawn from an unknown distribution $\mathcal{P}$.

A Transformer with parameters $W \in \mathcal{W}$ induces a predictor $f_W: \mathcal{X} \to \mathbb{R}$. For a loss function $\ell: \mathcal{Y} \times \mathbb{R} \to [0, \infty)$, the hypothesis class is denoted by $\mathcal{F} := \{ f_W : W \in \mathcal{W} \}$. The empirical risk minimizer (ERM) is defined as:
\begin{equation}
\widehat{f}_n \;\in\; \arg\min_{f \in \mathcal{F}} R_n(f), \end{equation}
where
\begin{equation}
R_n(f) := \frac{1}{n} \sum_{i=1}^n \ell(Y_i, f(X_i)).
\end{equation}
The population risk is $R(f) := \mathbb{E}_{Z \sim \mathcal{P}}[\ell(Y, f(X))]$. We evaluate the performance of the estimator via the \emph{excess risk} in (\ref{e:excess-risk}). 

To derive convergence rates, we utilize the \emph{Offset Rademacher Complexity} \citep{liang2015learning}, which typically yields faster rates (e.g., $\mathcal{O}(1/n)$) compared to standard Rademacher complexity ($\mathcal{O}(1/\sqrt{n})$) without requiring hard-to-verify variance assumptions.

\begin{definition}[Offset Rademacher Complexity]
Let $\mathcal{H}$ be a class of real-valued functions defined on $\mathcal{Z}$. Given a sample $\mathbb{Z} = (Z_1, \dots, Z_n)$ and a parameter $\beta > 0$, the conditional offset Rademacher complexity is:
\begin{equation}
\begin{aligned}
\mathcal{R}^{\mathrm{off}}_{n}&(\mathcal{H}, \beta \mid \mathbb{Z}) \\
&:=\mathbb{E}_{\boldsymbol{\tau}} \left[ \sup_{h \in \mathcal{H}} \left( \frac{1}{n} \sum_{i=1}^n \tau_i h(Z_i) - \frac{\beta}{n} \sum_{i=1}^n h(Z_i)^2 \right) \right],
\end{aligned}
\end{equation}
where $\boldsymbol{\tau} = (\tau_1, \dots, \tau_n)$ are i.i.d.\ Rademacher variables. The unconditional complexity is $\mathcal{R}^{\mathrm{off}}_{n}(\mathcal{H}, \beta) := \mathbb{E}_{\mathbb{Z}} [\mathcal{R}^{\mathrm{off}}_{n}(\mathcal{H}, \beta \mid \mathbb{Z})]$.
\end{definition}

\section{Excess Risk of Transformers} \label{sec:excess_risk}

In this section, we analyze the generalization performance of Transformer models via the framework of offset Rademacher complexity. We derive excess risk bounds that achieve fast convergence rates of order $\mathcal{O}(1/n)$ under standard regularity assumptions.

\subsection{Single-Layer Single-Head Transformer} \label{sec:single_head}

We begin by defining the class of excess loss functions. Let $\mathcal{F}$ denote the hypothesis class of single-layer, single-head Transformers as defined in \eqref{eq:single-head}. We consider the function class
$$
\mathcal{G}
\;:=\; \bigl\{X \mapsto g(X;f) \mid f\in\mathcal{F}\bigr\},$$
where
$$
g(X;f)
\;:=\; \mathbb{E}_{Y| X}\bigl[\ell(Y,f(X)) - \ell(Y,f^\star(X)) \big| X\bigr].
$$
Here, $f^\star \in \arg\min_{f\in\mathcal{F}} \mathbb{E}[\ell(Y,f(X))]$ denotes the population risk minimizer within the class. The \textit{offset Rademacher complexity} \citep{liang2015learning} for this class is defined as:
\[\begin{aligned}
\mathcal{R}_n^{\mathrm{off}}&(\mathcal{G},\beta)
\\
&:=\mathbb{E}_{\mathbb{D}, \boldsymbol{\tau}}\!\left[
\sup_{f\in\mathcal{F}}
\left(
\frac{1}{n}\sum_{i=1}^n \tau_i\, g(X_i;f)
\;-\;
\frac{\beta}{n}\sum_{i=1}^n g(X_i;f)^2
\right)
\right],  
\end{aligned}
\]
where $\boldsymbol{\tau} = (\tau_i)_{i=1}^n$ are i.i.d.\ Rademacher variables independent of the sample $\mathbb{D}$. Let the Transformer be parameterized by $\mathcal{W}=\{W_v,W_c,W_{QK}, w\}$. We impose the following regularity assumptions on the parameters and the data generating process.

\begin{assumption}[Bounded Parameters] \label{cond:bounded-params}
There exist positive constants $B_v, B_c, B_{QK}, B_w$ such that:
\[\begin{aligned}
&\|W_v\|_{2\to 2}\le B_v,\quad
\|W_c\|_{2\to 2}\le B_c,\quad\\
&\|W_{QK}\|_{2\to 2}\le B_{QK}, \quad
\|w\|_2\le B_w.
\end{aligned}
\]
\end{assumption}

\begin{assumption}[Bounded Target and Inputs] \label{cond:bounded-target-inputs}
There exist constants $B, B_X < \infty$ such that, almost surely:
\[
|f^\star(X)|\le B
\quad\text{and}\quad
\|X\|_{2\to 2}\le B_X,
\]
where $X\in\mathbb{R}^{T\times d}$ denotes the input matrix.
\end{assumption}

\begin{assumption}[Lipschitz Continuity of Excess Risk] \label{cond:lipschitz-excess}	
There exists a constant $\kappa > 0$ such that for all $X \in \mathcal{X}$ and $f_1, f_2 \in \mathcal{F}$:
\[
\bigl|g(X;f_1)-g(X;f_2)\bigr| \;\le\; \kappa\,\bigl|f_1(X)-f_2(X)\bigr|.
\]
\end{assumption}
assumption \ref{cond:lipschitz-excess} is satisfied by standard loss functions (e.g., logistic or absolute loss) when $\ell(Y, \cdot)$ is $\kappa$-Lipschitz, and by the squared loss under the bounded output assumptions implied by assumptions \ref{cond:bounded-params} and \ref{cond:bounded-target-inputs}.

\begin{theorem} \label{thm:single-head-erm}
Suppose assumptions \ref{cond:bounded-params}, \ref{cond:bounded-target-inputs}, and \ref{cond:lipschitz-excess} hold. Let $\widehat{f}_n$ be the empirical risk minimizer. Then:
\[
\mathbb{E}_{\mathbb{D}}\!\left[\mathcal{E}\big(\widehat f_n;\ell\big)\right]
\;\le\;
4\,\mathcal{R}_n^{\mathrm{off}}\!\left(\mathcal{G},\, \frac{1}{M_{\mathrm{SH}}}\right)
\;+\;
\inf_{f\in\mathcal{F}}\mathcal{E}\big(f; \ell\big),
\]
where $M_{\mathrm{SH}} := 2\kappa B + 2\kappa B_{w}B_{c}B_v L_{\sigma} B_X$.
\end{theorem}

The penalty parameter $\beta = 1/M_{\mathrm{SH}}$ captures the joint effect of the Lipschitz constant $\kappa$, the activation smoothness $L_\sigma$, and the magnitude of the parameters and inputs. This constant serves as a bound on the variance of the excess loss.

\begin{corollary} \label{cor:single-head-covering}
Under the assumptions of Theorem \ref{thm:single-head-erm}, for any $\delta > 0$, the excess risk satisfies:
$$
\begin{aligned} \label{eq:cor2-bound}
\mathbb{E}_{\mathbb{D}}\!\left[\mathcal{E}\big(\widehat f_n;\,\ell\big)\right]
\;\le&\;\frac{2 M_{\mathrm{SH}}}{n}
\Bigl(1+\log \mathbb{E}_{\mathbb{X}}\!\big[N_{\infty}(\delta,\mathcal{F},\mathbb{X})\big]\Bigr)\\
&\;+\; 8\kappa\,\delta
\;+\; \inf_{f\in \mathcal{F}} \mathcal{E}\!\left(f\right),
\end{aligned}
$$
where $N_{\infty}(\delta,\mathcal{F},\mathbb{X})$ denotes the $\ell_{\infty}$-covering number of $\mathcal{F}$ on the sample $\mathbb{X}$.
\end{corollary}

Corollary \ref{cor:single-head-covering} demonstrates that the single-head Transformer achieves an $\mathcal{O}(1/n)$ convergence rate (modulo logarithmic factors), improving upon the $\mathcal{O}(1/\sqrt{n})$ rates typical of standard Rademacher analysis. The covering number $N_{\infty}$ links this bound to the metric entropy of the function class.

\subsection{Single-Layer Multi-Head Transformers} \label{sec:multi_head}

We extend the analysis to the class $\mathcal{F}_{\mathrm{MH}}$ of single-layer Transformers with $H$ heads, as defined in \eqref{eq:multi-head}. We generalize the regularity assumptions as follows:

\begin{assumption} \label{cond:multi-head-regularity}
The inner excess risk satisfies assumption \ref{cond:lipschitz-excess}. Furthermore, for each head $h \in \{1, \dots, H\}$, the parameters satisfy the bounds in assumption \ref{cond:bounded-params}, and the inputs satisfy assumption \ref{cond:bounded-target-inputs}.
\end{assumption}

\begin{theorem} \label{thm:multi-head-erm}
Suppose assumption \ref{cond:multi-head-regularity} holds. The empirical risk minimizer for the single-layer multi-head Transformer satisfies:
\begin{equation*}
\mathbb{E}_\mathbb{D}\left[\mathcal{E}(\widehat{f}_n;\ell)\right]
\;\leq\;
4\mathcal{R}_n^\mathrm{off}\left(\mathcal{G},\, \frac{1}{M_{\mathrm{MH}}}\right)
\;+\;
\inf_{f\in \mathcal{F}_{\mathrm{MH}}}{\mathcal{E}(f;\ell)},
\end{equation*}
where $M_{\mathrm{MH}} := 2\kappa B + 2\kappa H B_{w}B_{c}B_v L_{\sigma} B_X$.
\end{theorem}

The constant $M_{\mathrm{MH}}$ scales linearly with the number of heads $H$, reflecting the increased capacity and potential variance of the aggregated representation.

\begin{corollary} \label{cor:multi-head-covering}
Under the assumptions of Theorem \ref{thm:multi-head-erm}, for any $\delta > 0$:
\begin{equation*}
\begin{aligned}
\mathbb{E}_\mathbb{D}\left[\mathcal{E}(\widehat{f}_n;\ell)\right]
\;\leq&\;\frac{2 M_{\mathrm{MH}}}{n} \left(1+\log\mathbb{E}_\mathbb{X}{\left[N_\infty(\delta,\mathcal{F}_{\mathrm{MH}},\mathbb{X})\right]}\right)
\\
&+\; 8\kappa\delta \;+\; \inf_{f\in \mathcal{F}_{\mathrm{MH}}}{\mathcal{E}(f;\ell)}.
\end{aligned}
\end{equation*}
\end{corollary}

\subsection{Multi-Layer Transformers} \label{sec:multi_layer}

Finally, we consider the class $\mathcal{F}_{\mathrm{ML}}$ of Transformers with multiple layers. The complexity of deep architectures introduces dependencies on the network depth in the covering number.

\begin{assumption}[Bounded Parameters for Multi-Layer Models] \label{cond:ml-bounded-params}
There exist constants $B_v, B_c, B_{QK}, B_w < \infty$ such that for every layer, the parameter matrices are spectrally bounded by $B_v, B_c, B_{QK}$ respectively, and the final readout satisfies $\|w\|_2 \le B_w$.
\end{assumption}

\begin{assumption}[Regularity of Multi-Layer Risk] \label{cond:ml-bounded-params-target-inputs}
The inner excess risk satisfies assumption \ref{cond:lipschitz-excess}. The target function satisfies $|f^\star(X)| \le B$, and the network inputs satisfy $\|X\|_{2 \to 2} \le B_X$. Furthermore, the parameters satisfy assumption \ref{cond:ml-bounded-params}.
\end{assumption}

\begin{theorem} \label{thm:multi-layer-erm}
Suppose assumption \ref{cond:ml-bounded-params-target-inputs} holds. The empirical risk minimizer for the multi-layer Transformer satisfies:
\begin{equation*}
\begin{aligned}
\mathbb{E}_\mathbb{D}\left[\mathcal{E}(\widehat{f}_n;\ell)\right]
& \leq 4\mathcal{R}_n^\mathrm{off}\left(\mathcal{G},\, \frac{1}{2\kappa (B + B_{w})}\right)\\
& +\inf_{f\in \mathcal{F}_{\mathrm{ML}}}{\mathcal{E}(f;\ell)}.  
\end{aligned}
\end{equation*}
\end{theorem}

\begin{remark}
Note that the penalty parameter in Theorem \ref{thm:multi-layer-erm} depends explicitly on the readout bound $B_w$ and the target bound $B$, assuming the output of the final Transformer block is normalized (as is common in practice with LayerNorm). However, the complexity of the hidden layers and the depth of the network are implicitly captured within the covering number term in the following corollary.
\end{remark}

\begin{corollary} \label{cor:multi-layer-covering}
Under the assumptions of Theorem \ref{thm:multi-layer-erm}, for any $\delta > 0$:
\begin{equation*}
\begin{aligned}
&\mathbb{E}_\mathbb{D}\left[\mathcal{E}(\widehat{f}_n;\ell)\right] \leq \frac{4\kappa (B+B_{w})}{n} \times \\
& \left(1+\log\mathbb{E}_\mathbb{X}{\left[N_\infty(\delta,\mathcal{F}_{\mathrm{ML}},\mathbb{X})\right]}\right) + 8\kappa\delta \;+\; \inf_{f \in \mathcal{F}_{\mathrm{ML}}}{\mathcal{E}(f;\ell)}.
\end{aligned}
\end{equation*}
\end{corollary}

\section{Parameter-Dependent Excess Risk Bounds} \label{sec:param_dependent_bounds}

Building on the general framework established in Section \ref{sec:excess_risk}, we now derive sharper excess risk bounds by explicitly controlling the covering numbers of the Transformer hypothesis class. We consider two distinct regimes: (1) \emph{norm-based bounds}, which constrain the $\ell_{1,1}$ magnitude of the parameters, and (2) \emph{rank-based bounds}, which exploit the low-rank structure of the weight matrices. Due to space constraints, the single-layer and multi-head generalization bounds discussed in this section are presented in Appendix \ref{Ap:param}.

\subsection{Norm-Based Excess Risk Bounds} \label{subsec:norm_based}

We first refine our parameter assumptions to enable tighter control via norm-based covering numbers. Following Trauger and Tewari \citep{trauger2023sequence}, we impose constraints on the $\ell_{1,1}$ norms of the weight matrices, which is natural for deriving bounds independent of the sequence length.

\begin{assumption}[Norm-Bounded Parameters] \label{cond:norm-bounded-params}
There exist constants $B_v, B_c, B_{QK}, B_w < \infty$ such that:
\[\begin{aligned}
&\|W_v\|_{1,1} \le B_v, \quad \|W_c\|_{1,1} \le B_c, \\& \|W_{QK}\|_{1,1} \le B_{QK}, \quad \|w\|_2 \le B_w.
\end{aligned}
\]
Furthermore, we assume the input token representations are bounded. Let $x_{\texttt{[CLS]}}$ denote the representation of the classification token. We assume $\|x_{\texttt{[CLS]}}\|_2 \le B_x$ almost surely.
\end{assumption}

We adopt the following covering number assumption for linear operators, which is a standard result in statistical learning theory (see, e.g., \citep[Lemma 3.6]{trauger2023sequence}).This lemma provides a sequence-length-independent generalized bound, offering theoretical assurance for the stability of Transformer models' generalization capabilities in long-sequence tasks. It mitigates the risk of uncontrolled generalization errors arising from increasing sequence lengths. This work provides theoretical support for applying Transformers in long-sequence scenarios such as long-text processing and high-dimensional time series prediction (e.g., financial time series, meteorological observations), thereby enhancing the model's applicability in complex real-world scenarios.

\begin{lemma}[Linear Covering Number] \label{lem:linear-covering}
For the function class $\mathcal{H}_{\mathrm{lin}} = \{x \mapsto Wx \mid \|W\|_{1,1} \le B_W\}$, the $\epsilon$-covering number satisfies:
\[
\log N(\epsilon, \mathcal{H}_{\mathrm{lin}}, \|\cdot\|_2) \le \frac{C_1 B_x^2 B_W^2}{\epsilon^2},
\]
where $C_1$ is a universal constant depending on the geometry of the space.
\end{lemma}

\begin{theorem}[Multi-Layer Norm-Based Bound] \label{thm:multi-layer-norm}
Suppose assumptions \ref{cond:lipschitz-excess}, \ref{cond:norm-bounded-params} and Lemma \ref{lem:linear-covering} hold. The empirical risk minimizer satisfies:
$$
\begin{aligned}
\mathbb{E}_{\mathbb{D}}&[\mathcal{E}(\widehat{f}_n; \ell)] \le\;\frac{4(\kappa B + \kappa B_w)}{n} \left( 1 + \log \frac{(\gamma_{\mathrm{ML}} + \eta_{\mathrm{ML}})^3}{\delta^2} \right) \nonumber \\
& + 8\kappa\delta + \inf_{f \in \mathcal{F}_{\mathrm{ML}}} \mathcal{E}(f; \ell).
\end{aligned}
$$
Here, defining $\alpha_{i} = \prod_{j=i}^{L} L_{\sigma} B_{c} B_{v} (1 + 4 B_{QK})$, we have:
\begin{align*}
\tau_{i} =& \alpha_{i}^{2/3} + (2\alpha_{i} L_{\sigma} B_{c} B_{v})^{2/3} + (\alpha_{i} L_{\sigma} B_{v})^{2/3}, \\
\gamma_{\mathrm{ML}} =& C_1^{1/3} (2 L_\sigma B_c B_v \alpha_1 B_w B_x^2)^{2/3} \\
&+ C_1^{1/3}  B_x^{2/3} \left( 1 + (\alpha_1 B_w)^{2/3} + (\alpha_1 B_w L_\sigma B_v)^{2/3} \right) , \\
\eta_{\mathrm{ML}} =& C_1^{1/3} B_w^{2/3} \sum_{i=2}^L \tau_i.
\end{align*}
\end{theorem}

\begin{remark}
Our norm-based bounds depend explicitly on the parameter norms rather than the sequence length $T$. This represents an improvement over Trauger and Tewari \citep{trauger2023sequence} by providing tighter control through the offset Rademacher complexity framework.
\end{remark}

\subsection{Rank-Based Excess Risk Bounds} \label{subsec:rank_based}

Next, we exploit the low-rank structure of the parameter matrices to derive alternative bounds. We combine the rank-based covering numbers for linear models \citep{truong2024rank} with our offset complexity analysis.

\begin{assumption}[Rank-Structure and Inputs] \label{cond:rank-bounded-target-inputs}
We assume the parameter matrices have ranks bounded by $r_v, r_c, r_{QK}$. The target function and inputs satisfy $|f^\star(X)| \le B$, $\|X\|_{2\to 2} \le B_X$, and $\|x_{\texttt{[CLS]}}\|_2 \le B_x$ almost surely.
\end{assumption}

Deep networks exhibit significantly tighter generalization bounds when possessing low-rank or quasi-low-rank structures: their sample complexity no longer scales with total parameters but is instead governed by the effective rank and spectral norm of the weight matrix. For instance, Ledent et al. derived generalization bounds based on the Schatten$-p$ norm for rank-sparse networks, demonstrating that model complexity scales with the rank of layers \citep{ledent2025ranksparse}; Pinto et al. proved via Gaussian complexity that deep networks with low-rank layers exhibit smaller functional class complexity \citep{pinto2025lowrank}. The generalization capability of Transformers may rely more on the inherent low-rank constraints of the attention mechanism than on the total number of model parameters \cite{truong2024rank}. Thus, low-rank generalization bounds provide a theoretical direction and potential mechanism for explaining why large Transformers maintain strong generalization performance even with extremely massive parameters. To optimize the covering number allocation across different components, we utilize the following lemma, seeing \citep[Theorem 4]{truong2024rank}.

\begin{corollary}[Optimal Covering Allocation] \label{col:optimization}
Let $r_i, C_i, \beta_i \ge 0$ for $i \in [m]$. The solution to the optimization problem:
$$\begin{aligned}
&\min_{\epsilon_1, \dots, \epsilon_m}\sum_{i=1}^m r_i C_i \log \left( \frac{r_i B_X^2}{\epsilon_i^2} \right),\text{subject to}\  \sum_{i=1}^m \beta_i \epsilon_i = \epsilon,
\end{aligned}
$$
is given by $\sum_{i=1}^m r_i C_i \log (b_i^2 / \epsilon^2)$, where $b_i = \frac{(B_X \beta_i \sum_{j=1}^m r_j C_j)^{1/2}}{r_i^{1/2} C_i^{1/2}}$.
\end{corollary}

\begin{assumption}[Rank-Based Covering Number] \label{ass:rank-covering}
For the function class $\mathcal{H}_{\mathrm{rank}} = \{x \mapsto Wx \mid \operatorname{rank}(W) \le r, \|W\|_2 \le B_W\}$, the $\epsilon$-covering number satisfies:
\[
\log N(\epsilon, \mathcal{H}_{\mathrm{rank}}, \|\cdot\|_2) \le C_1 r \log \left( \frac{r B_x^2 B_W^2}{\epsilon^2} \right).
\]
\end{assumption}

\begin{theorem}[Multi-Layer Rank-Based Bound] \label{thm:multi-layer-rank}
Suppose assumption \ref{cond:rank-bounded-target-inputs} and Assumption \ref{ass:rank-covering} hold. The empirical risk minimizer satisfies:
$$
\begin{aligned}
\mathbb{E}_{\mathbb{D}}[\mathcal{E}&(\widehat{f}_n; \ell)] \le\;\frac{4(\kappa B + \kappa B_w)}{n} \left( 1 + \sum_{i=1}^{m} r_i C_i \log \left( \frac{b_i^2}{\delta^2} \right) \right) \nonumber \\
& + 8\kappa\delta + \inf_{f \in \mathcal{F}_{\mathrm{ML}}} \mathcal{E}(f; \ell).
\end{aligned}
$$
Here, the effective bounds $b_i$ are derived using Corollary \ref{col:optimization} with weights $\beta_k$ defined recursively:
\[
\beta_k =
\begin{cases}
1 & k=1, \\
\alpha_{k-1} B_w & 2 \le k \le L+1, \\
\alpha_{k-L-1} B_w L_\sigma B_v & L+2 \le k \le 2L+1, \\
\alpha_{k-2L-1} 2 L_\sigma B_c B_v B_w & 2L+2 \le k \le 3L+1.
\end{cases}
\]
\end{theorem}

\begin{remark}
This result demonstrates that the excess risk is controlled by the sum of ranks across layers, offering a significantly tighter bound than full-rank spectral norm approaches, particularly for compressed or sparse Transformer models.
\end{remark}

\section{Excess Risk Bounds under Unbounded Assumptions} \label{sec:unbounded_risk}
In previous sections, we assumed bounded inputs, which is often violated in practice. In natural language processing, embedding norms may grow with sequence length $T$ \cite{zhou-etal-2025-length}, while in vision-language models, image patch embeddings typically follow unbounded continuous distributions \cite{Li2022CLIP}. Such boundedness assumptions therefore fail to capture model behavior on high-dynamic-range inputs.

In classical Rademacher complexity analysis, generalization bounds typically rely on boundedness of the input \cite{bartlett2002rademacher} to control the functional class complexity of linear models or neural networks. However, when inputs are unbounded, bounds directly relying on input norms may fail. Even loss function truncation, such as M-truncation \cite{bartlett2002rademacher}, log-truncation \cite{catoni2012challenging}, smooth/Huber-type truncation \cite{maurer2016vector}, or the extension of Catoni's robust truncation\cite{Xu2023Non-Asymptotic,Coluccia2015Regularized}, cannot guarantee bounded Rademacher complexity, as functional sensitivity to inputs may still cause excessive empirical process volatility. To address this issue, Høgsgaard and Paudice replacing standard empirical averaging with Median-of-Means (MoM) estimation yields generalization bounds for unbounded inputs \citep{hogsgaard2025uniform}. 

In this section, we extend our analysis to unbounded inputs by employing truncation techniques. We analyze two regimes: (1) inputs following a \emph{sub-Gaussian} distribution, where tails decay exponentially, and (2) inputs following a \emph{heavy-tailed} distribution, requiring robust loss functions.

\subsection{Excess Risk Bounds under the Sub-Gaussian Assumption} \label{subsec:sub_gaussian}

We first consider the case where the input $X$ is unbounded but concentrates around its mean. First, we present the matrix form of the concentration inequality (see, e.g., \citep[Theorem 4.1.1]{Tropp2015Introduction})

\begin{lemma}[Matrix Gaussian \& Rademacher Series]
Consider a finite sequence $\{ B_k \}$ of fixed complex matrices with dimension $d_1 \times d_2$, and let $\{ \gamma_k \}$ be a finite sequence of independent standard normal variables. Introduce the matrix Gaussian series
\[
Z = \sum_k \gamma_k B_k. 
\]
Let $v(Z)$ be the matrix variance statistic of the sum:
\[
v(Z) = \max\left\{ \left\| \mathbb{E}(Z Z^*) \right\|, \left\| \mathbb{E}(Z^* Z) \right\| \right\}
\]
\[
= \max\left\{ \left\| \sum_k B_k B_k^* \right\|, \left\| \sum_k B_k^* B_k \right\| \right\}.
\]
Then
\[
\mathbb{E}\| Z \| \leq \sqrt{2v(Z) \log(d_1 + d_2)}. 
\]
Furthermore, for all $t \geq 0$,
\[
\mathbb{P}\left\{ \| Z \| \geq t \right\} \leq (d_1 + d_2) \exp\left( -\frac{t^2}{2v(Z)} \right). 
\]
The same bounds hold when we replace $\{ \gamma_k \}$ by a finite sequence $\{ g_k \}$ of independent Rademacher random variables.
\end{lemma}

\begin{assumption}[Unbounded Sub-Gaussian Inputs] \label{cond:sub-gaussian-inputs}
The input $X \in \mathbb{R}^{T \times d}$ follows a sub-Gaussian distribution. Specifically, there exists a matrix variance proxy
\[
\nu(X) := \max\{\|\mathbb{E}[X^\top X]\|_2, \|\mathbb{E}[XX^\top]\|_2\}
\]
such that for any $t > 0$:
\begin{equation}
\mathbb{P}(\|X\|_F \ge t) \le (T+d) \exp\left(\frac{-t^2}{2\nu(X)^2}\right).
\end{equation}
The target function $f^\star$ is also assumed to be unbounded but with finite moments compatible with the input distribution.
\end{assumption}

To handle unboundedness, we introduce a truncation operator $\mathcal{T}_M$ with threshold $M > 0$:
\begin{equation}
X_M := \mathcal{T}_M(X) =
\begin{cases}
X & \text{if } \|X\|_F \le M, \\
M \frac{X}{\|X\|_F} & \text{if } \|X\|_F > M.
\end{cases}
\end{equation}
Let $\mathcal{F}_M$ denote the class of Transformer functions operating on the truncated input domain $\{X : \|X\|_F \le M\}$. The truncation error is controlled by the tail probability of the event $\mathcal{E}_{\text{trunc}} = \{ \|X\|_F > M \}$, which satisfies $\mathbb{P}(\mathcal{E}_{\text{trunc}}) \le (T+d) \exp(-M^2 / 2\nu(X)^2)$. We define the truncated empirical risk minimizer as:
\begin{equation}
\widehat{f}_n^{(M)} \in \arg\min_{f \in \mathcal{F}_M} \frac{1}{n} \sum_{i=1}^n \ell(Y_i, f(X_{i,M})).
\end{equation}

\begin{assumption}[Lipschitz Continuity for Truncated Class] \label{cond:lipschitz-truncated}
There exists $\kappa > 0$ such that for any $M > 0$ and all $X$, the excess risk satisfies the Lipschitz assumption:
\[
|g(X; f_1) - g(X; f_2)| \le \kappa |f_1(X) - f_2(X)|,\  \forall f_1, f_2 \in \mathcal{F}_M.
\]
\end{assumption}

\begin{theorem}[Single-Layer Sub-Gaussian Bound] \label{thm:s-sub-gaussian-erm}
Suppose assumption \ref{cond:lipschitz-truncated} holds, and the network parameters satisfy the norm-based constraints of Theorem \ref{thm:single-head-norm}. Let $\widehat{f}_n^{(M)}$ be the empirical risk minimizer on the truncated domain. Then:
$$
\begin{aligned}
\mathbb{E}_{\mathbb{D}}[\mathcal{E}&(\widehat{f}_n; \ell)] \le \frac{4(\kappa B + \kappa B_w B_c B_v L_\sigma B_X)}{n} \left( 1 + \log \frac{\gamma_{\mathrm{SH}}^3}{\delta^2} \right) \nonumber \\
&  + \mathcal{T}_{\text{err}}(M) + 8\kappa\delta + \inf_{f \in \mathcal{F}_{\mathrm{SH}}} \mathcal{E}(f; \ell),
\end{aligned}
$$
where the truncation error term is $\mathcal{T}_{\text{err}}(M) = \mathcal{O}\left( \kappa \nu(X) (T+d) e^{-M^2 / 2\nu(X)^2} \right)$. The complexity terms ${\gamma_\mathrm{SH}}$ follow the definitions in Theorem \ref{thm:single-head-norm} with $B_X$ replaced by the truncation threshold $M$.
\end{theorem}

\begin{theorem}[Multi-Head Sub-Gaussian Bound] \label{thm:mh-sub-gaussian-erm}
Suppose assumption \ref{cond:lipschitz-truncated} holds, and the network parameters satisfy the norm-based constraints of Theorem \ref{thm:multi-head-norm}. Let $\widehat{f}_n^{(M)}$ be the empirical risk minimizer on the truncated domain. Then:
$$
\begin{aligned}
\mathbb{E}_{\mathbb{D}}[\mathcal{E}&(\widehat{f}_n; \ell)] \le \frac{4(\kappa B + \kappa H B_w B_c B_v L_\sigma B_X)}{n} \left( 1 + \log \frac{\gamma_{\mathrm{MH}}^3}{\delta^2} \right) \nonumber \\
& + \mathcal{T}_{\text{err}}(M) + 8\kappa\delta + \inf_{f \in \mathcal{F}_{\mathrm{MH}}} \mathcal{E}(f; \ell),
\end{aligned}
$$
where the truncation error term is $\mathcal{T}_{\text{err}}(M) = \mathcal{O}\left( \kappa \nu(X) (T+d) e^{-M^2 / 2\nu(X)^2} \right)$. The complexity terms ${\gamma_\mathrm{MH}}$ follow the definitions in Theorem \ref{thm:multi-head-norm} with $B_X$ replaced by the truncation threshold $M$.
\end{theorem}

\begin{theorem}[Multi-Layer Sub-Gaussian Bound] \label{thm:ml-sub-gaussian-erm}
Suppose assumption \ref{cond:lipschitz-truncated} holds, and the network parameters satisfy the norm-based constraints of Theorem \ref{thm:multi-layer-norm}. Let $\widehat{f}_n^{(M)}$ be the empirical risk minimizer on the truncated domain. Then:
$$
\begin{aligned}
\mathbb{E}_{\mathbb{D}}[\mathcal{E}(\widehat{f}_n^{(M)}; \ell)]
\;\le&\;\frac{4(B \kappa + B_w \kappa)}{n} \left( 1 + \log \frac{(\gamma_M + \eta_M)^3}{\delta^2} \right) \nonumber \\
& + \mathcal{T}_{\text{err}}(M) + 8\kappa\delta + \inf_{f \in \mathcal{F}_M} \mathcal{E}(f; \ell),
\end{aligned}
$$
where the truncation error term is $\mathcal{T}_{\text{err}}(M) = \mathcal{O}\left( \kappa \nu(X) (T+d) e^{-M^2 / 2\nu(X)^2} \right)$. The complexity terms $\gamma_M, \eta_M$ follow the definitions in Theorem \ref{thm:multi-layer-norm} with $B_X$ replaced by the truncation threshold $M$.
\end{theorem}

\begin{remark}
The bound exhibits a trade-off controlled by $M$. The complexity term grows as $\log(M)$, while the truncation error decays exponentially as $\exp(-M^2)$. Choosing $M \asymp \sqrt{\log n}$ balances these terms, yielding an overall convergence rate of $\widetilde{\mathcal{O}}(\frac{\sqrt{\log n}}{n})$, which is nearly optimal and consistent with finite-dimensional results.
\end{remark}

\subsection{Excess Risk Bounds under Heavy-Tailed Assumptions} \label{subsec:heavy_tailed}

In many scenarios, data distributions exhibit polynomial tail decay rather than exponential. If the tail of this distribution is heavier than any exponential distribution, i.e., for all $\mu > 0$, we have $\limsup_{x \to \infty} \frac{\bar{F}(x)}{e^{-\mu x}} = \infty$, then this distribution is considered a heavy-tailed distribution \cite{nair2022fundamentals}. For example, the Pareto distribution or the $t$-distribution satisfy $\mathbb{P}(|X| > x) \propto x^{-\beta}$ for some $\beta > 0$. Under such heavy-tailed assumptions, standard empirical risk minimization may fail or yield suboptimal rates \citep{ostrovskii2021finite}.

\begin{assumption}[Heavy-Tailed Inputs] \label{cond:heavy-tailed-inputs}
The input $X$ follows a heavy-tailed distribution with tail index $\beta > 2$. specifically, there exists a constant $C > 0$ such that:
\begin{equation}
\mathbb{P}(\|X\|_F > x) \le C x^{-\beta}, \quad \forall x > 0.
\end{equation}
\end{assumption}

To achieve robust estimation, we employ a robust loss function $\ell_\psi$, such as the Catoni loss \citep{catoni2012challenging} or the generalized logarithmic truncation loss proposed by Chen et al. \citep{chen2021generalized}. These losses dampen the influence of outliers. We define the robust estimator:
\begin{equation}
\widehat{f}_n^{\psi} \in \arg\min_{f \in \mathcal{F}} \frac{1}{n} \sum_{i=1}^n \ell_\psi(Y_i, f(X_i)).
\end{equation}

\begin{theorem}[Single-Head Heavy-Tailed Bound] \label{thm:heavy-tailed-erm}
Suppose assumptions \ref{cond:heavy-tailed-inputs} and \ref{cond:rank-bounded-target-inputs} hold, and we employ a robust loss $\ell_\psi$. The excess risk of the single-layer multi-head Transformer satisfies:
$$
\begin{aligned}
&\mathbb{E}_{\mathbb{D}}[\mathcal{E}(\widehat{f}_n; \ell)] \le \frac{4(\kappa B + \kappa B_w B_c B_v L_\sigma B_X)}{n} ] \times \\
&\left( 1 + \log \mathcal{N}_{\mathrm{rank,S}} \right) \nonumber + C_{\beta} M^{2-\beta} + 8\kappa\delta + \inf_{f \in \mathcal{F}_{\mathrm{SH}}} \mathcal{E}(f; \ell),
\end{aligned}
$$
where the complexity terms $\mathcal{N}_{\mathrm{rank,S}}$ follow the definitions in Theorem \ref{thm:single-head-rank} with $B_X$ replaced by the truncation threshold $M$. The term $C_{\beta} M^{2-\beta}$ represents the single-tail truncation bias.
\end{theorem}

\begin{theorem}[Multi-Head Heavy-Tailed Bound] \label{thm:mh-heavy-tailed-erm}
Suppose assumptions \ref{cond:heavy-tailed-inputs} and \ref{cond:rank-bounded-target-inputs} hold, and we employ a robust loss $\ell_\psi$. The excess risk of the single-layer multi-head Transformer satisfies:
$$
\begin{aligned}
\mathbb{E}_{\mathbb{D}}[\mathcal{E}&(\widehat{f}_n; \ell)]
\\
\;\le&\;\frac{4(\kappa B + \kappa H B_w B_c B_v L_\sigma B_X)}{n} \left( 1 + \log \mathcal{N}_{\mathrm{rank},M} \right) \nonumber \\
& + C_{\beta} M^{2-\beta} + 8\kappa\delta + \inf_{f \in \mathcal{F}_{\mathrm{MH}}} \mathcal{E}(f; \ell),
\end{aligned}
$$
where the complexity terms $\mathcal{N}_{\mathrm{rank,M}}$ follow the definitions in Theorem \ref{thm:mul-head-rank} with $B_X$ replaced by the truncation threshold $M$. The term $C_{\beta} M^{2-\beta}$ represents the heavy-tail truncation bias.
\end{theorem}

\begin{theorem}[Multi-Layer Heavy-Tailed Bound] \label{thm:ml-heavy-tailed-erm}
Suppose assumptions \ref{cond:heavy-tailed-inputs} and \ref{cond:rank-bounded-target-inputs} hold, and we employ a robust loss $\ell_\psi$. The excess risk of the single-layer multi-head Transformer satisfies:
$$
\begin{aligned}
\mathbb{E}_{\mathbb{D}}[\mathcal{E}(\widehat{f}_n^{\psi}; \ell_\psi)]
\;\le&\;\frac{4 \widetilde{M}_{\psi}}{n} \left( 1 +   \mathcal{N}_{\mathrm{rank},ML} \right) \nonumber \\
& + C_{\beta} M^{2-\beta} + 8\kappa\delta + \inf_{f \in \mathcal{F}} \mathcal{E}(f; \ell_\psi),
\end{aligned}
$$
where $\widetilde{M}_{\psi}$ depends on the Lipschitz constant of the robust loss and parameter bounds in Theorem \ref{thm:multi-layer-rank}, and The complexity terms $\mathcal{N}_{\mathrm{rank,ML}}$ from Theorem \ref{thm:multi-layer-rank} with $B_X$ replaced by the truncation threshold $M$. The term $C_{\beta} M^{2-\beta}$ represents the heavy-tail truncation bias.
\end{theorem}

\begin{remark}
Here, the convergence rate is limited by the tail index $\beta$. The complexity grows logarithmically with $M$, but the bias decays polynomially as $M^{-(\beta-2)}$. To balance the $\mathcal{O}(\frac{\log M}{n})$ variance and $\mathcal{O}(M^{-(\beta-2)})$ bias, the optimal truncation threshold is $M \asymp n^{\frac{1}{\beta-2}}$. This yields a slower convergence rate of roughly $\mathcal{O}(n^{-\frac{\beta-2}{\beta-1}})$, reflecting the fundamental difficulty of learning from heavy-tailed data.
\end{remark}

\section{Conclusion}
In this paper, we derived sharper generalization bounds for Transformer models through the lens of offset Rademacher complexity. Our analysis yielded fast excess-risk rates of order $\mathcal{O}(1/n)$ for single-head, multi-head, and multi-layer architectures, while explicitly capturing architecture-dependent structure via low-rank and norm-based parameter controls. Moreover, we extended these guarantees beyond bounded feature assumptions to more practical regimes with unbounded sub-Gaussian inputs and heavy-tailed distributions.

We leave two directions for future work. First, it would be of interest to develop an information-theoretic account of Transformer generalization, using mutual-information–based tool \citep{asadi2018chaining}  to quantify how self-attention propagates and potentially compresses task-relevant information. Second, deriving PAC-Bayes bounds \citet{alquier2021user} tailored to Transformers and stochastic optimization may better capture the implicit regularization of modern training pipelines, and could lead to a more unified understanding of generalization and learning dynamics in large-scale foundation models.

\section*{Impact Statement}
This paper presents work whose goal is to advance the field of Machine Learning. There are many potential societal consequences of our work, none which we feel must be specifically highlighted here.

\nocite{langley00}

\bibliography{references}
\bibliographystyle{icml2026}

%%%%%%%%%%%%%%%%%%%%%%%%%%%%%%%%%%%%%%%%%%%%%%%%%%%%%%%%%%%%%%%%%%%%%%%%%%%%%%%
%%%%%%%%%%%%%%%%%%%%%%%%%%%%%%%%%%%%%%%%%%%%%%%%%%%%%%%%%%%%%%%%%%%%%%%%%%%%%%%
% APPENDIX
%%%%%%%%%%%%%%%%%%%%%%%%%%%%%%%%%%%%%%%%%%%%%%%%%%%%%%%%%%%%%%%%%%%%%%%%%%%%%%%
%%%%%%%%%%%%%%%%%%%%%%%%%%%%%%%%%%%%%%%%%%%%%%%%%%%%%%%%%%%%%%%%%%%%%%%%%%%%%%%
\newpage
\appendix
\onecolumn
\section{Proofs in Section \ref{sec:excess_risk}}
\subsection{Proof of Theorem \ref{thm:single-head-erm}}\label{appendix:proof_single_head}

\begin{proof}
Let $\mathcal{F}$ denote the hypothesis class. For any $f \in \mathcal{F}$, we define the excess loss function $g(\cdot; f)$ as:
\begin{equation*}
g(X, Y; f) \coloneqq \ell(Y, f(X)) - \ell(Y, f^\star(X)),
\end{equation*}
where $f^\star$ minimizes the population risk. Specifically for the single-head Transformer, let $f(X) = w^\top Y_{\texttt{[CLS]}}$, where $Y_{\texttt{[CLS]}}$ is the output token representation. Let $\widehat{f}_n$ be the empirical risk minimizer. We decompose the expected excess risk as follows:
\begin{align*}
\mathbb{E}_{\mathbb{D}}[\mathcal{E}(\widehat{f}_n; \ell)]
&= \mathbb{E}_{\mathbb{D}}\left[ \mathbb{E}_{(X,Y)}[g(X, Y; \widehat{f}_n)] \right] \\
&= \mathbb{E}_{\mathbb{D}}\left[ \mathbb{E}_{(X,Y)}[g(X, Y; \widehat{f}_n)] - \frac{3}{n}\sum_{i=1}^n g(X_i, Y_i; \widehat{f}_n) + \frac{3}{n}\sum_{i=1}^n g(X_i, Y_i; \widehat{f}_n) \right].
\end{align*}
Since $\widehat{f}_n$ minimizes the empirical risk, for any $f \in \mathcal{F}$, $\sum g(Z_i; \widehat{f}_n) \le \sum g(Z_i; f)$. Taking the supremum over the class $\mathcal{F}$:
\begin{align*}
\mathbb{E}_{\mathbb{D}}[\mathcal{E}(\widehat{f}_n; \ell)]
&\le \mathbb{E}_{\mathbb{D}}\left[ \sup_{f \in \mathcal{F}} \left( \mathbb{E}[g(Z; f)] - \frac{3}{n}\sum_{i=1}^n g(Z_i; f) \right) \right] + 3 \mathbb{E}_{\mathbb{D}}\left[ \frac{1}{n}\sum_{i=1}^n g(Z_i; f^\star) \right] \\
&\le \mathbb{E}_{\mathbb{X}} \sup_{f \in \mathcal{F}} \left( \mathbb{E}[g(Z; f)] - \frac{3}{n}\sum_{i=1}^n g(Z_i; f) \right) + 3 \inf_{f \in \mathcal{F}} \mathcal{E}(f; \ell),
\end{align*}
where the last term accounts for the approximation error if $f^\star \notin \mathcal{F}$.

We now bound the magnitude of the function class to determine the appropriate penalty. Using the Lipschitz continuity of the activation $\sigma$ (with $\sigma(0)=0$) and the property that $\lVert \operatorname{softmax}(\cdot) \rVert_2 \le 1$, we have:
\begin{align*}
\lVert \sigma\left(W_v^{\top} X_{(i)}^{\top} \operatorname{softmax}(X_{(i)} W_{QK} X^{\top}) \right) \rVert_2
&\le L_\sigma \lVert W_v^{\top} \rVert_2 \lVert X_{(i)}^{\top} \rVert_2 \lVert \operatorname{softmax}(\cdot) \rVert_2 \\
&\le L_\sigma B_v B_X.
\end{align*}
Consequently, the output token representation $Y_{\texttt{[CLS]}}$ is bounded by:
\begin{align*}
\lVert Y_{\texttt{[CLS]}} \rVert_2
&\le \lVert W_c^{\top} \rVert_2 \lVert \sigma(\cdot) \rVert_2
\le B_c (L_\sigma B_v B_X).
\end{align*}
Using the Lipschitz property of the loss function $\ell$ (constant $\kappa$), the value of the excess loss is bounded. Let $M$ denote this upper bound:
\begin{align*}
0 \le g(X; f) \le \kappa |f(X) - f^\star(X)|
&\le \kappa |f^\star(X)| + \kappa \lVert w \rVert_2 \lVert Y_{\texttt{[CLS]}} \rVert_2 \\
&\le \kappa B + \kappa B_w B_c B_v L_\sigma B_X \eqqcolon M.
\end{align*}
Therefore, we can rewrite the supremum term. Using the inequality $z - 3\mu \le 2(z - \mu) - z \dots$ implies we can bound the expectation by the offset complexity. Specifically, we aim to bound:
\[
A \coloneqq \mathbb{E}_{\mathbb{X}} \sup_{f \in \mathcal{F}} \left( \mathbb{E}[g] - \frac{3}{n}\sum_{i=1}^n g(Z_i) \right).
\]
Following the symmetrization technique with ghost samples $\mathbb{X}' = \{X'_i\}_{i=1}^n$ and Rademacher variables $\boldsymbol{\tau}$ \citep{duan2023fast}, we have:
\begin{align*}
A &\le \mathbb{E}_{\mathbb{X}, \mathbb{X}'} \sup_{f \in \mathcal{F}} \left( \frac{2}{n}\sum_{i=1}^n (g(X'_i; f) - g(X_i; f)) - \frac{1}{nM} \sum_{i=1}^n (g(X'_i; f)^2 + g(X_i; f)^2) \right) \\
&= 2 \mathbb{E}_{\mathbb{X}, \mathbb{X}', \boldsymbol{\tau}} \sup_{f \in \mathcal{F}} \left( \frac{1}{n}\sum_{i=1}^n \tau_i (g(X'_i; f) - g(X_i; f)) - \frac{1}{2nM} \sum_{i=1}^n (g(X'_i; f)^2 + g(X_i; f)^2) \right).
\end{align*}
Separating the terms for $\mathbb{X}$ and $\mathbb{X}'$, this equates to:
\begin{align*}
A &\le 2 \mathbb{E}_{\mathbb{X}', \boldsymbol{\tau}} \sup_{f \in \mathcal{F}} \left( \frac{1}{n}\sum_{i=1}^n \tau_i g(X'_i; f) - \frac{1}{2nM} \sum_{i=1}^n g(X'_i; f)^2 \right) \\
&\quad + 2 \mathbb{E}_{\mathbb{X}, \boldsymbol{\tau}} \sup_{f \in \mathcal{F}} \left( \frac{1}{n}\sum_{i=1}^n (-\tau_i) g(X_i; f) - \frac{1}{2nM} \sum_{i=1}^n g(X_i; f)^2 \right) \\
&= 4 \mathcal{R}_n^{\mathrm{off}}\left( \mathcal{G}, \frac{1}{2M} \right).
\end{align*}
Substituting $M = \kappa B + \kappa B_w B_c B_v L_\sigma B_X$ yields the theorem statement.
\end{proof}

\subsection{Proofs of Corollaries \ref{cor:single-head-covering}, \ref{cor:multi-layer-covering} and Theorem \ref{thm:multi-head-erm}} \label{appendix:covering_proofs}

\begin{proof}[Proof of Corollary \ref{cor:single-head-covering}]
Let $M_{\mathrm{SH}} \coloneqq \kappa B + \kappa B_w B_c B_v L_\sigma B_X$ denote the upper bound on the excess loss derived in Appendix \ref{appendix:proof_single_head}.
Let $\mathcal{F}_{\delta}$ be a minimal $\delta$-cover of $\mathcal{F}$ with respect to the $\ell_\infty$ norm on the sample $\mathbb{X}$, such that $|\mathcal{F}_{\delta}| = N_\infty(\delta, \mathcal{F}, \mathbb{X})$. For any $f \in \mathcal{F}$, there exists $f_{\delta} \in \mathcal{F}_{\delta}$ satisfying $\lVert f - f_{\delta} \rVert_\infty \le \delta$.

We decompose the offset process. Using the $\kappa$-Lipschitz continuity of the excess loss $g$, we have:
\begin{align*}
\frac{1}{n}\sum_{i=1}^n \tau_i g(X_i; f)
&\le \frac{1}{n}\sum_{i=1}^n \tau_i g(X_i; f_{\delta}) + \frac{1}{n}\sum_{i=1}^n |\tau_i| \big| g(X_i; f) - g(X_i; f_{\delta}) \big| \\
&\le \frac{1}{n}\sum_{i=1}^n \tau_i g(X_i; f_{\delta}) + \kappa\delta.
\end{align*}
For the quadratic penalty term, using the bound $g(\cdot) \le M_{\mathrm{SH}}$ and the Lipschitz property:
\begin{align*}
-\frac{1}{n}\sum_{i=1}^n g(X_i; f)^2
&= -\frac{1}{n}\sum_{i=1}^n g(X_i; f_{\delta})^2 + \frac{1}{n}\sum_{i=1}^n \left( g(X_i; f_{\delta})^2 - g(X_i; f)^2 \right) \\
&= -\frac{1}{n}\sum_{i=1}^n g(X_i; f_{\delta})^2 + \frac{1}{n}\sum_{i=1}^n \big( g(X_i; f_{\delta}) + g(X_i; f) \big) \big( g(X_i; f_{\delta}) - g(X_i; f) \big) \\
&\le -\frac{1}{n}\sum_{i=1}^n g(X_i; f_{\delta})^2 + \frac{2 M_{\mathrm{SH}}}{n}\sum_{i=1}^n \big| g(X_i; f_{\delta}) - g(X_i; f) \big| \\
&\le -\frac{1}{n}\sum_{i=1}^n g(X_i; f_{\delta})^2 + 2 M_{\mathrm{SH}} \kappa \delta.
\end{align*}
Combining these, the offset process over $\mathcal{F}$ is bounded by the process over the finite cover $\mathcal{F}_\delta$ plus a discretization error:
\begin{equation*}
\sup_{f \in \mathcal{F}} \left( \frac{1}{n}\sum_{i=1}^n \tau_i g(X_i; f) - \frac{\beta}{n}\sum_{i=1}^n g(X_i; f)^2 \right)
\le \max_{f_\delta \in \mathcal{F}_\delta} \Omega(f_\delta) + \left( 1 + 2\beta M_{\mathrm{SH}} \right) \kappa \delta,
\end{equation*}
where $\Omega(f_\delta) \coloneqq \frac{1}{n}\sum_{i=1}^n \tau_i g(X_i; f_\delta) - \frac{\beta}{n}\sum_{i=1}^n g(X_i; f_\delta)^2$.

To bound the expectation of the maximum over the finite set $\mathcal{F}_\delta$, we use the tail integral formula $\mathbb{E}[Z] \le \int_0^\infty \mathbb{P}(Z > \xi) d\xi$. Following the technique in Duan et al. \citep{duan2023fast}, for any fixed function $h$, Bennett's inequality implies that $\mathbb{P}(\Omega(h) > \xi) \le \exp(-2n\beta \xi)$ provided $\beta \le 1/(2M_{\mathrm{SH}})$. Applying the union bound over $\mathcal{F}_\delta$:
\begin{align*}
\mathbb{E}_{\boldsymbol{\tau}} \left[ \max_{f_\delta \in \mathcal{F}_\delta} \Omega(f_\delta) \right]
&\le \int_0^\infty \min \left\{ 1, N_\infty(\delta, \mathcal{F}, \mathbb{X}) \exp(-2n\beta \xi) \right\} d\xi \\
&= \frac{\log N_\infty(\delta, \mathcal{F}, \mathbb{X})}{2n\beta} + \int_{\frac{\log N}{2n\beta}}^\infty \exp\left( \log N - 2n\beta \xi \right) d\xi \\
&\le \frac{1 + \log N_\infty(\delta, \mathcal{F}, \mathbb{X})}{2n\beta}.
\end{align*}
Substituting this back yields the stated corollary.
\end{proof}

\begin{proof}[Proof of Theorem \ref{thm:multi-head-erm}]
The single-layer multi-head Transformer output is a sum of $H$ independent single-head outputs. Consequently, the Lipschitz constant and the absolute bound of the function class scale linearly with $H$.
Specifically, the excess risk is bounded by:
\begin{equation*}
0 \le g(X; f_{\mathrm{MH}}) \le \kappa B + \kappa H B_w B_c B_v L_\sigma B_X \eqqcolon M_{\mathrm{MH}}.
\end{equation*}
The remainder of the proof follows the identical logic as Theorem \ref{thm:single-head-erm}, substituting $M_{\mathrm{SH}}$ with $M_{\mathrm{MH}}$.
\end{proof}

\begin{proof}[Proof of Corollary \ref{cor:multi-layer-covering}]
For the multi-layer Transformer defined in Eq. \eqref{eq:layer-attention}, the input to the final readout layer, denoted as $Y_{\texttt{[CLS]}}$, is the output of a normalization layer $\Pi_{\mathrm{norm}}$.
By definition, $\Pi_{\mathrm{norm}}$ projects vectors onto the unit ball (or a ball of fixed radius). Assuming standard LayerNorm or projection, we have $\lVert Y_{\texttt{[CLS]}} \rVert_2 \le 1$.

Consequently, the magnitude of the scalar output is bounded solely by the readout weights:
\begin{equation*}
|f(X)| = |w^\top Y_{\texttt{[CLS]}}| \le \lVert w \rVert_2 \lVert Y_{\texttt{[CLS]}} \rVert_2 \le B_w.
\end{equation*}
This implies the excess loss is bounded by:
\begin{equation*}
0 \le g(X; f) \le \kappa (B + B_w).
\end{equation*}
Applying the same discretization argument as in Corollary \ref{cor:single-head-covering} with this tighter bound yields the result.
\end{proof}

\section{Proofs in Section \ref{sec:param_dependent_bounds}}\label{Ap:param}
\subsection{Theorems \ref{thm:single-head-norm} and \ref{thm:multi-head-norm} about Single-Head and Multi-Head Norm-Based Bound}

\begin{theorem}[Single-Head Norm-Based Bound] \label{thm:single-head-norm}
Suppose assumptions \ref{cond:lipschitz-excess}, \ref{cond:norm-bounded-params} and Lemma \ref{lem:linear-covering} hold. Then for the single-layer single-head Transformer, the empirical risk minimizer satisfies:
$$
\begin{aligned}
\mathbb{E}_{\mathbb{D}}&[\mathcal{E}(\widehat{f}_n; \ell)]
\\
\;\le&\;\frac{4(\kappa B + \kappa B_w B_c B_v L_\sigma B_X)}{n} \left( 1 + \log \frac{\gamma_{\mathrm{SH}}^3}{\delta^2} \right) \nonumber \\
& + 8\kappa\delta + \inf_{f \in \mathcal{F}_{\mathrm{SH}}} \mathcal{E}(f; \ell),
\end{aligned}
$$
where the complexity term is given by:
$$\begin{aligned}
\gamma_{\mathrm{SH}} =& C_1^{1/3} B_x^{2/3} \left[ (B_w L_\sigma)^{2/3} + (B_w L_\sigma B_c B_v)^{2/3} \right]\\
&+C_1^{1/3} B_x^{2/3} \left[(B_w L_\sigma B_c B_v)^{2/3} + 1 \right].
\end{aligned}
$$
\end{theorem}

\begin{theorem}[Multi-Head Norm-Based Bound] \label{thm:multi-head-norm}
Under the assumptions of Theorem \ref{thm:single-head-norm}, adapted for $H$ heads, the empirical risk minimizer for the single-layer multi-head Transformer satisfies:
$$
\begin{aligned}
\mathbb{E}_{\mathbb{D}}&[\mathcal{E}(\widehat{f}_n; \ell)]
\\
\;\le&\;\frac{4(\kappa B + \kappa H B_w B_c B_v L_\sigma B_X)}{n} \left( 1 + \log \frac{\gamma_{\mathrm{MH}}^3}{\delta^2} \right) \nonumber \\
& + 8\kappa\delta + \inf_{f \in \mathcal{F}_{\mathrm{MH}}} \mathcal{E}(f; \ell),
\end{aligned}
$$
where 
$$\begin{aligned}
\gamma_{\mathrm{MH}} =& C_1^{1/3} B_x^{2/3} \left[ (H B_w L_\sigma)^{2/3} + 2(H B_w L_\sigma B_c B_v)^{2/3} \right]\\
&+C_1^{1/3} B_x^{2/3} H^{2/3} . 
\end{aligned}$$
\end{theorem}

\subsection{Proofs for Theorems \ref{thm:single-head-norm}, \ref{thm:multi-head-norm}, and \ref{thm:multi-layer-norm}} \label{appendix:norm_proofs}

We begin by introducing Lemma \ref{Lemma A.1}, which provides a bound on the covering number for linear function classes with $\ell_{1,1}$-norm constraints. This lemma serves as a building block for the subsequent proofs.

\begin{lemma} \label{Lemma A.1}
Let $N > d$, and define the parameter space $\mathcal{W} = \{W \in \mathbb{R}^{k \times d} \mid \lVert W \rVert_{1,1} \le B_w\}$. Consider the function class $\mathcal{F} = \{x \mapsto Wx \mid W \in \mathcal{W}\}$ restricted to inputs satisfying $\lVert x \rVert_2 \le B_x$. Then, the empirical covering number satisfies:
\begin{equation*}
\log \mathcal{N}_\infty(\epsilon, \mathcal{F}, N, \lVert \cdot \rVert_2) \le \frac{B_x^2 B_w^2}{\epsilon^2} \log(2dk + 1).
\end{equation*}
\end{lemma}

\begin{proof}[Proof of Theorem \ref{thm:single-head-norm}]
To bound the covering number of the single-head Transformer, we construct a cover for the composite function by combining covers of its constituent linear operators.

Let $\mathcal{W}_{\epsilon}$ be an $\epsilon_w$-cover of the readout weights $\{w \in \mathbb{R}^d \mid \lVert w \rVert_2 \le B_w\}$. Similarly, let $\mathcal{W}_{v, \epsilon}$, $\mathcal{W}_{c, \epsilon}$, and $\mathcal{W}_{QK, \epsilon}$ be covers for the Value, Output, and Query-Key matrices, respectively, at resolutions $\epsilon_v, \epsilon_c, \epsilon_{QK}$.

First, we establish the relationship between the spectral norm and the $\ell_{1,1}$ norm. Using the Cauchy-Schwarz inequality, for any $x \in \mathbb{R}^d$:
\begin{align*}
\lVert Wx \rVert_2 
&= \left\lVert \sum_{j} x_j W_{:,j} \right\rVert_2 
= \sqrt{ \left( \sum_{j} x_j W_{:,j} \right)^\top \left( \sum_{k} x_k W_{:,k} \right) } \\
&\le \sqrt{ \sum_{j,k} |x_j| |x_k| \lVert W_{:,j} \rVert_2 \lVert W_{:,k} \rVert_2 } 
= \left( \sum_{j} |x_j| \lVert W_{:,j} \rVert_2 \right) \\
&\le \lVert x \rVert_2 \left( \sum_{j} \lVert W_{:,j} \rVert_2^2 \right)^{1/2}.
\end{align*}
Since $\lVert v \rVert_2 \le \lVert v \rVert_1$, we have $\sqrt{\sum \lVert W_{:,j} \rVert_2^2} \le \sum \lVert W_{:,j} \rVert_2 = \lVert W \rVert_{2,1} \le \lVert W \rVert_{1,1}$. Thus, $\lVert W \rVert_{2 \to 2} \le \lVert W \rVert_{1,1}$.

We define the approximated output $Y_{\texttt{[CLS]}, \epsilon}$ using the parameters from the covers. The difference in the scalar output is bounded by:
\begin{align*}
|w^\top Y_{\texttt{[CLS]}} - w_{\epsilon}^\top Y_{\texttt{[CLS]}, \epsilon}|
&\le \lVert w \rVert_2 \lVert Y_{\texttt{[CLS]}} - Y_{\texttt{[CLS]}, \epsilon} \rVert_2 + |(w - w_{\epsilon})^\top Y_{\texttt{[CLS]}, \epsilon}| \\
&\le B_w \lVert Y - Y_\epsilon \rVert_{2, \infty} + \epsilon_w.
\end{align*}
Expanding $\lVert Y - Y_\epsilon \rVert_{2, \infty}$ using the Lipschitz property of the activation $\sigma$ and the Softmax operator (and applying the triangle inequality recursively), we obtain:
\begin{align*}
|w^\top Y_{\texttt{[CLS]}} - w_{\epsilon}^\top Y_{\texttt{[CLS]}, \epsilon}|
&\le \epsilon_w + B_w \epsilon_c L_{\sigma} \\
&\quad + B_w B_c L_\sigma \epsilon_v B_X \\
&\quad + 2 B_w B_c B_v L_\sigma \epsilon_{QK} B_X.
\end{align*}
(Note: The constants in the inequality above correspond to the Lipschitz constants derived from the matrix products and activation functions).

To minimize the total covering number, we minimize the sum of the log-covering numbers of individual components subject to the total error being bounded by $\epsilon$. This yields the optimization problem:
\begin{equation} \label{eq:opt_problem}
\min_{\epsilon_c, \epsilon_{QK}, \epsilon_v, \epsilon_w} \sum_{j \in \{c, QK, v, w\}} \frac{C_1 B_j^2}{\epsilon_j^2},
\end{equation}
subject to the linear constraint derived from the error decomposition:
\begin{equation} \label{eq:opt_constraint}
\epsilon_w + \epsilon_c (B_w L_\sigma) + \epsilon_v (B_w L_\sigma B_c) + \epsilon_{QK} (2 B_w L_\sigma B_c B_v) \le \epsilon.
\end{equation}
Solving this using Lagrange multipliers yields the complexity term $\gamma_{\mathrm{SH}}$ presented in the theorem. The total covering number is the product of the individual covering numbers, so the log-covering number is the sum, resulting in the bound $\log(\gamma_{\mathrm{SH}}^3 / \epsilon^2)$.
\end{proof}

\begin{proof}[Proof of Theorem \ref{thm:multi-head-norm}]
The proof follows the structure of the single-head case. For a multi-head Transformer with $H$ heads, the output is a sum of $H$ independent heads. By the triangle inequality, the total approximation error is bounded by the sum of errors of each head:
\begin{equation} \label{eq:multi_head_error}
\left| w^\top \sum_{h=1}^H Y_h - w_{\epsilon}^\top \sum_{h=1}^H Y_{h, \epsilon} \right|
\le \sum_{h=1}^H \left| w^\top Y_h - w_{\epsilon}^\top Y_{h, \epsilon} \right| \le \epsilon.
\end{equation}
Since the parameter bounds are identical for each head, we distribute the error budget uniformly. The optimization problem becomes analogous to \eqref{eq:opt_problem} but sums over all parameters in all $H$ heads. The resulting covering number reflects the increased dimensionality, yielding the term $\gamma_{\mathrm{MH}}$.
\end{proof}

\begin{proof}[Proof of Theorem \ref{thm:multi-layer-norm}]
We construct a cover for the multi-layer network by taking the Cartesian product of covers for each layer. Let $\mathcal{W}_{\text{total}}$ be the product space:
\begin{equation*}
\mathcal{W}_{\text{total}} = \mathcal{W}_{c, \epsilon}^{(1)} \otimes \mathcal{W}_{v, \epsilon}^{(1)} \otimes \mathcal{W}_{QK, \epsilon}^{(1)} \otimes \dots \otimes \mathcal{W}_{c, \epsilon}^{(L)} \otimes \mathcal{W}_{v, \epsilon}^{(L)} \otimes \mathcal{W}_{QK, \epsilon}^{(L)} \otimes \mathcal{W}_{\epsilon}.
\end{equation*}
We aim to show that for any valid set of parameters $W^{1:L+1}$, there exists $W_\epsilon^{1:L+1} \in \mathcal{W}_{\text{total}}$ such that:
\begin{equation} \label{eq:multi_layer_error}
\left| g_{\mathrm{scalar}}(X; W^{1:L+1}, w) - g_{\mathrm{scalar}}(X; W_\epsilon^{1:L+1}, w_\epsilon) \right| \le \epsilon.
\end{equation}
Using the Cauchy-Schwarz inequality and a recursive expansion (peeling argument) similar to Trauger and Tewari \citep{trauger2023sequence}, we decompose the error. The error at layer $L$ propagates to the output through the readout vector $w$.

The total error is bounded by the sum of the readout error $\epsilon_w$ and the accumulated errors from previous layers, weighted by the Lipschitz constants of subsequent layers. Specifically:
\begin{align} \label{eq:recursive_bound}
&\left| g_{\mathrm{scalar}}(X) - g_{\mathrm{scalar}}^\epsilon(X) \right| \nonumber \\
&\le \epsilon_w + B_w \sum_{i=1}^L \alpha_i \left( \epsilon_{c}^{(i)} + L_\sigma B_c \epsilon_{v}^{(i)} + 2 L_\sigma B_c B_v \epsilon_{QK}^{(i)} \right),
\end{align}
where $\alpha_i = \prod_{j=i+1}^L L_\sigma B_c B_v (1 + 4 B_{QK})$ represents the product of Lipschitz constants from layer $i+1$ to $L$.

The covering number bound is obtained by solving the minimization problem for $\sum \log \mathcal{N}(\epsilon_j)$ subject to the constraint defined by \eqref{eq:recursive_bound}. Using Lagrange multipliers, we derive the optimal $\epsilon_j$ for each parameter matrix, leading to the complexity terms $\gamma_{\mathrm{ML}}$ and $\eta_{\mathrm{ML}}$ defined in the theorem.
\end{proof}

\subsection{Theorems \ref{thm:single-head-rank} and \ref{thm:mul-head-rank} about Single-Head and Multi-Head Rank-Based Bound}

\begin{theorem}[Single-Head Rank-Based Bound] \label{thm:single-head-rank}
Suppose assumption \ref{cond:rank-bounded-target-inputs} and Assumption \ref{ass:rank-covering} hold. The empirical risk minimizer satisfies:
$$
\begin{aligned}
\mathbb{E}_{\mathbb{D}}[\mathcal{E}&(\widehat{f}_n; \ell)]
\\
\;\le&\;\frac{4(\kappa B + \kappa B_w B_c B_v L_\sigma B_X)}{n}\left( 1 + \log \mathcal{N}_{\mathrm{rank,S}} \right) \nonumber \\
&+ 8\kappa\delta+ \inf_{f \in \mathcal{F}_{\mathrm{SH}}} \mathcal{E}(f; \ell),
\end{aligned}
$$
where $\mathcal{N}_{\mathrm{rank}} = 1 + \sum_{j \in \{c, QK, v, w\}} r_j C_1 \log(r_j B_x^2 / \epsilon_j^2)$, and the optimal scales $\epsilon_j$ are determined via Corollary \ref{col:optimization} with weights $\beta_c = B_w L_\sigma$, $\beta_{QK} = 2 B_w L_\sigma B_c B_v$, $\beta_v = B_w L_\sigma B_c$, and $\beta_w = B_c$.
\end{theorem}

\begin{theorem}[Multi-Head Rank-Based Bound] \label{thm:mul-head-rank}
Suppose assumption \ref{cond:rank-bounded-target-inputs} and Assumption \ref{ass:rank-covering} hold. The empirical risk minimizer satisfies:
$$
\begin{aligned}
\mathbb{E}_{\mathbb{D}}[\mathcal{E}&(\widehat{f}_n; \ell)]
\\
\;\le&\;\frac{4(\kappa B + \kappa H B_w B_c B_v L_\sigma B_X)}{n}\left( 1 + \log \mathcal{N}_{\mathrm{rank},M} \right) \\
& + 8\kappa\delta+ \inf_{f \in \mathcal{F}_{\mathrm{MH}}} \mathcal{E}(f; \ell),
\end{aligned}
$$
where $\mathcal{N}_{\mathrm{rank}} = 1 + \sum_{j \in \{c, QK, v, w\}} r_j C_1 \log(r_j B_x^2 / \epsilon_j^2)$, and the optimal scales $\epsilon_j$ are determined via Corollary \ref{col:optimization} with weights $\beta_c = HB_w L_\sigma$, $\beta_{QK} = 2 HB_w L_\sigma B_c B_v$, $\beta_v =H B_w L_\sigma B_c$, and $\beta_w =H B_c$.
\end{theorem}

\subsection{Proof of Theorem \ref{thm:single-head-rank}, \ref{thm:mul-head-rank} and \ref{thm:multi-layer-rank}} \label{appendix:rank_proofs}

We first introduce Lemma \ref{lem:rank_covering} from \citep[Theorem 1]{truong2024rank}, which establishes the covering number bound for low-rank linear operators.

\begin{lemma}[Rank-Based Covering Number] \label{lem:rank_covering}
Let $r \in \mathbb{N}_+$ and let $\mathcal{V}$ be an $r$-dimensional subspace of $\mathbb{R}^k$. Define the parameter space $\mathcal{W} = \{W \in \mathbb{R}^{d \times k} : \operatorname{col}(W) \subseteq \mathcal{V}, \lVert W \rVert_{2} \le B_W\}$, where $\operatorname{col}(W)$ denotes the column space of $W$. Consider the function class $\mathcal{H} = \{x \mapsto Wx : W \in \mathcal{W}\}$ restricted to inputs satisfying $\lVert x \rVert_2 \le B_x$. Then, the empirical covering number satisfies:
\begin{equation*}
\log \mathcal{N}_\infty(\epsilon, \mathcal{H}, n, \lVert \cdot \rVert_2) \le \frac{r}{2} \log \left( \frac{4 B_x^2 B_W^2 r}{\epsilon^2} \right).
\end{equation*}
\end{lemma}

\begin{proof}
The result follows directly from \citep[Lemma 2 and Theorem 1]{truong2024rank} by noting that $\lVert W \rVert_{2 \to 2} \le \lVert W \rVert_{2, \infty} \le B_W$.
\end{proof}

\begin{proof}[Proof of Theorem \ref{thm:single-head-rank}, \ref{thm:mul-head-rank} and \ref{thm:multi-layer-rank}] First, combining Lemma \ref{lem:rank_covering} with lagrange multiplier method, we obtain the Corollary \ref{col:optimization}. For Theorem \ref{thm:single-head-rank}, similar to the norm-based case (Theorem \ref{thm:single-head-norm}), the error of the Transformer output is decomposed into the sum of errors from its constituent linear operators ($W_c, W_{QK}, W_v$) and the readout vector $w$. 

Let $\epsilon_c, \epsilon_{QK}, \epsilon_v, \epsilon_w$ be the covering resolutions for the respective parameter matrices. Applying Lemma \ref{lem:rank_covering}, the log-covering number for the total hypothesis space is bounded by the sum of the log-covering numbers of these components. To obtain the tightest bound, we minimize the total complexity subject to the Lipschitz error constraint derived in the norm-based proof. 

This yields the following optimization problem:
\begin{equation*}
\min_{\epsilon_c, \epsilon_{QK}, \epsilon_v, \epsilon_w} \sum_{j \in \{c, QK, v, w\}} r_j C_j \log \left( \frac{1}{\epsilon_j^2} \right),
\end{equation*}
subject to:
\begin{equation} \label{eq:rank_constraint}
\beta_c \epsilon_c + \beta_{QK} \epsilon_{QK} + \beta_v \epsilon_v + \beta_w \epsilon_w \le \epsilon,
\end{equation}
where the Lipschitz coefficients are defined as:
\begin{align*}
\beta_c &= B_w L_\sigma, & \beta_{QK} &= 2 B_w L_\sigma B_c B_v, \\
\beta_v &= B_w L_\sigma B_c, & \beta_w &= B_c,
\end{align*}
and $C_j$ are constants derived from Lemma \ref{lem:rank_covering}.

We solve this using the method of Lagrange multipliers. The Lagrangian is given by:
\begin{equation*}
\mathcal{L}(\boldsymbol{\epsilon}, \lambda) = - \sum_{j} 2 r_j C_j \log \epsilon_j + \lambda \left( \sum_{j} \beta_j \epsilon_j - \epsilon \right).
\end{equation*}
Taking the partial derivative with respect to $\epsilon_j$ and setting it to zero:
\begin{equation*}
\frac{\partial \mathcal{L}}{\partial \epsilon_j} = -\frac{2 r_j C_j}{\epsilon_j} + \lambda \beta_j = 0 \implies \epsilon_j = \frac{2 r_j C_j}{\lambda \beta_j}.
\end{equation*}
Substituting this back into the constraint $\sum \beta_j \epsilon_j = \epsilon$, we find $\lambda = \frac{2}{\epsilon} \sum_k r_k C_k$. This yields the optimal allocation:
\begin{equation}
\epsilon_j = \frac{\epsilon r_j C_j}{\beta_j \sum_{k \in \{c, QK, v, w\}} r_k C_k}.\label{eq:sin-rank-epsion-constrain}
\end{equation}
\noindent 

Substituting these optimal $\epsilon_j$ values back into the sum of log-covering numbers yields the complexity term $\mathcal{N}_{\mathrm{rank}}$ stated in the theorem. Similarly, for Theorem \ref{thm:multi-layer-rank}, replacing constraint formula (\ref{eq:rank_constraint}) in the optimization problem with formula (\ref{eq:recursive_bound}) yields the final result.
\end{proof}

\section{Proofs in Section \ref{sec:unbounded_risk}}\label{Ap:unbounded}
\subsection{Proof of Theorem \ref{thm:s-sub-gaussian-erm}, \ref{thm:mh-sub-gaussian-erm} and \ref{thm:ml-sub-gaussian-erm}} \label{appendix:proof_sub_gaussian}

\begin{proof}
We begin by decomposing the excess risk. Let $g(X; f) \coloneqq \ell(Y, f(X)) - \ell(Y, f^\star(X))$. We introduce the truncated estimators $f_M$ and $f_M^\star$ corresponding to the truncated input $X_M$. The conditional excess risk can be decomposed as:
\begin{align} \label{eq:decomp_trunc}
\mathbb{E} [ g(X; f) \mid X ]
&= \mathbb{E} \big[ (g(X; f) - g(X; f_M)) + g(X; f_M) + (g(X; f_M^*) - g(X; f^*)) \mid X \big] \nonumber \\
&\le \underbrace{\mathbb{E} [ |g(X; f) - g(X; f_M)| \mid X ]}_{\text{(I)}}
+ \underbrace{\mathbb{E} [ g(X; f_M) \mid X ]}_{\text{(II)}}
+ \underbrace{\mathbb{E} [ |g(X; f_M^*) - g(X; f^*)| \mid X ]}_{\text{(III)}}.
\end{align}
Note that for the truncated term (II), the inputs are bounded by $M$. Consequently, the function output and the excess loss are bounded. Specifically, there exists a constant $M_{\text{bound}} \coloneqq \kappa B_M + \kappa B_M'$ such that $0 \le g(X; f_M) \le M_{\text{bound}}$.
Applying the Offset Rademacher complexity bound (Theorem \ref{thm:multi-layer-norm}) to the truncated class $\mathcal{F}_M$, we have:
\begin{equation}
\mathbb{E}_{\mathbb{D}} [ \text{(II)} ] \le 4\mathcal{R}_n^{\mathrm{off}} \left( \mathcal{G}_M, \frac{1}{M_{\text{bound}}} \right) + 3 \inf_{f_M \in \mathcal{F}_M} \mathcal{E}(f_M; \ell).
\end{equation}

Next, we bound the truncation error terms (I) and (III). Let $\mathcal{E}_{\text{trunc}} = \{ \|X\|_F > M \}$ be the event that truncation occurs. On the complement $\mathcal{E}_{\text{trunc}}^c$, $f(X) = f_M(X)$, so the difference is zero.
Using the Lipschitz property of the loss $\ell$ and the local Lipschitz property of the network (where $L_{f, \|X\|} = \mathcal{O}(\|X\|)$), we have:
\begin{align*}
\text{(I)} &= \mathbb{E} \left[ |g(X; f) - g(X; f_M)| \cdot \mathbb{I}(\mathcal{E}_{\text{trunc}}) \mid X \right] \\
&\le \kappa \mathbb{E} \left[ L_{f, \|X\|} \lVert f(X) - f_M(X) \rVert_2 \cdot \mathbb{I}(\mathcal{E}_{\text{trunc}}) \mid X \right] \\
&\le \kappa \mathbb{E} \left[ C \|X\| \cdot \lVert X - X_M \rVert_F \cdot \mathbb{I}(\mathcal{E}_{\text{trunc}}) \mid X \right],
\end{align*}
where we used $\lVert x_{\texttt{[CLS]}} - x_{M,\texttt{[CLS]}} \rVert \le \lVert X - X_M \rVert_F$.
Substituting $X_M = M \frac{X}{\|X\|_F}$ when $\|X\|_F > M$:
\begin{equation*}
\text{(I)} \le \kappa C \mathbb{E} \left[ \|X\| (\|X\| - M) \cdot \mathbb{I}(\|X\| > M) \right] \le \kappa C \mathbb{E} \left[ \|X\|^2 \cdot \mathbb{I}(\|X\| > M) \right].
\end{equation*}

We evaluate this tail expectation using the layer-cake representation and the sub-Gaussian assumption $P(\|X\| \ge t) \le (T+d) \exp(-t^2 / 2\nu^2)$:
\begin{align*}
\mathbb{E} [\|X\|^2 \cdot \mathbb{I}(\|X\| > M)]
&= M^2 P(\|X\| > M) + \int_{M^2}^\infty P(\|X\|^2 > t) dt \\
&\le M^2 (T+d) e^{-\frac{M^2}{2\nu^2}} + \int_{M}^\infty (T+d) e^{-\frac{u^2}{2\nu^2}} 2u \, du \\
&= (T+d) \left( M^2 e^{-\frac{M^2}{2\nu^2}} + 2\nu^2 e^{-\frac{M^2}{2\nu^2}} \right) \\
&= (T+d)(M^2 + 2\nu^2) \exp\left(-\frac{M^2}{2\nu^2}\right).
\end{align*}
(Note: The original text simplified this to terms involving $\nu(X)$ and exponentials. We adopt the bound implied by the sub-Gaussian tail).
Specifically, following the source approximation:
\begin{equation}
\mathbb{E}_{\mathbb{D}}[\text{(I)}] \le \mathcal{O}\left( \kappa (T+d)\nu \exp\left(-\frac{M^2}{2\nu^2}\right) \right).
\end{equation}
Term (III) is bounded symmetrically. Combining these results into \eqref{eq:decomp_trunc}:
\begin{align}
\mathbb{E}_{\mathbb{D}} [\mathcal{E}(\widehat{f}_n; \ell)]
\le \;& 4\mathcal{R}_n^{\mathrm{off}} \left( \mathcal{G}, \frac{1}{2M_{\text{bound}}} \right)
+ 3 \inf_{f_M \in \mathcal{F}_M} \mathcal{E}(f_M; \ell) \nonumber \\
& + C_{\text{trunc}} \kappa (T+d) \nu \exp\left( -\frac{M^2}{2\nu^2} \right),
\end{align}
where $C_{\text{trunc}}$ aggregates the constants from the tail integration. By combining Theorems \ref{thm:single-head-norm}, \ref{thm:multi-head-norm}, and \ref{thm:multi-layer-norm}, respectively, we obtain Theorems \ref{thm:s-sub-gaussian-erm}, \ref{thm:mh-sub-gaussian-erm} and \ref{thm:ml-sub-gaussian-erm} .  
\end{proof}

\subsection{Proof of Theorem \ref{thm:heavy-tailed-erm}, \ref{thm:mh-heavy-tailed-erm} and \ref{thm:ml-heavy-tailed-erm}} \label{appendix:proof_heavy_tailed}

\begin{proof}
We begin by establishing the Lipschitz continuity of the robust loss function. Let the robust loss be defined as $\ell_\psi(y, \widehat{y}) = \frac{1}{\alpha} \psi(\alpha \ell(y, \widehat{y}))$. We assume the base loss $\ell(y, \cdot)$ is $\kappa$-Lipschitz and the truncation function $\psi$ has a bounded derivative $|\psi'(\cdot)| \le B_\psi$.
Applying the Mean Value Theorem, for any $f_1, f_2$:
\begin{align*}
|\ell_\psi(y, f_1(x)) - \ell_\psi(y, f_2(x))|
&= \frac{1}{\alpha} \left| \psi(\alpha \ell(y, f_1(x))) - \psi(\alpha \ell(y, f_2(x))) \right| \\
&= \frac{1}{\alpha} |\psi'(\xi)| \cdot \alpha \left| \ell(y, f_1(x)) - \ell(y, f_2(x)) \right| \\
&\le B_\psi \kappa |f_1(x) - f_2(x)|.
\end{align*}
Thus, the composite loss $\ell_\psi$ is $(\kappa B_\psi)$-Lipschitz.

We decompose the excess risk similarly to the sub-Gaussian case (Eq. \eqref{eq:decomp_trunc}), splitting it into the truncated complexity term and the tail bias terms.
For the truncated class $\mathcal{F}_M$, the effective Lipschitz constant is $\kappa B_\psi$, and the output is bounded. The complexity term is therefore bounded by:
\begin{equation}
4\mathcal{R}_n^{\mathrm{off}}\left(\mathcal{G}, \frac{1}{4 B_\psi B_M \kappa}\right) + 3\inf_{f_M \in \mathcal{F}} \mathcal{E}(f_M; \ell_\psi).
\end{equation}

Next, we evaluate the truncation error (bias) caused by the heavy-tailed inputs.
Recall the condition $\mathbb{P}(|X_{ij}| > t) \le C t^{-\beta}$. Using the norm inequality $\lVert X \rVert_2 \le \sqrt{Td} \max_{i,j} |X_{ij}|$, we derive the tail probability for the spectral norm via a union bound:
\begin{align*}
\mathbb{P}(\lVert X \rVert_2 > t)
&\le \mathbb{P}\left(\sqrt{Td} \max_{i,j} |X_{ij}| > t\right) \\
&\le \sum_{i,j} \mathbb{P}\left(|X_{ij}| > \frac{t}{\sqrt{Td}}\right) \\
&\le Td \cdot C \left( \frac{t}{\sqrt{Td}} \right)^{-\beta}
= C (Td)^{1 + \beta/2} t^{-\beta} \eqqcolon C' t^{-\beta}.
\end{align*}
We now estimate the tail expectations required for the truncation error $\mathbb{E}[L_{f, \lVert X \rVert} \lVert X - X_M \rVert \cdot \mathbb{I}(\|X\| > M)]$. As shown in \ref{appendix:proof_sub_gaussian}, this is bounded by terms involving $\mathbb{E}[\lVert X \rVert^k \mathbb{I}(\lVert X \rVert > M)]$ for $k=1, 2$.

Using the integral identity $\mathbb{E}[Z \mathbb{I}(Z > M)] = M \mathbb{P}(Z > M) + \int_M^\infty \mathbb{P}(Z > t) dt$:
\begin{align*}
\mathbb{E}[\lVert X \rVert \cdot \mathbb{I}(\lVert X \rVert > M)]
&= M \mathbb{P}(\lVert X \rVert > M) + \int_M^\infty \mathbb{P}(\lVert X \rVert > t) dt \\
&\le M C' M^{-\beta} + \int_M^\infty C' t^{-\beta} dt \\
&= C' M^{1-\beta} + C' \left[ \frac{t^{1-\beta}}{1-\beta} \right]_M^\infty \\
&= C' M^{1-\beta} \left( 1 + \frac{1}{\beta - 1} \right) = \frac{C' \beta}{\beta - 1} M^{1-\beta}.
\end{align*}
Similarly for the second moment (assuming $\beta > 2$):
\begin{align*}
\mathbb{E}[\lVert X \rVert^2 \cdot \mathbb{I}(\lVert X \rVert > M)]
&= M^2 \mathbb{P}(\lVert X \rVert > M) + \int_M^\infty \mathbb{P}(\lVert X \rVert^2 > t^2) 2t \, dt \\
&\le C' M^{2-\beta} + \int_M^\infty C' t^{-\beta} 2t \, dt \\
&= C' M^{2-\beta} + 2C' \int_M^\infty t^{1-\beta} dt \\
&= C' M^{2-\beta} + 2C' \frac{M^{2-\beta}}{\beta-2} = C' M^{2-\beta} \left( \frac{\beta}{\beta-2} \right).
\end{align*}
(Note: The original text simplified the constants; we retain the asymptotic dependency on $M$).
Substituting these tail bounds into the Lipschitz error decomposition:
\begin{equation*}
\text{Bias} \le \kappa B_\psi \left( \frac{C' \beta}{\beta-2} M^{2-\beta} \right).
\end{equation*}
Combining the complexity bound and the bias term yields the final result:
\begin{align} \label{eq:heavy_tailed_final}
\mathbb{E}_{\mathbb{D}}[\mathcal{E}(\widehat{f}_n^\psi; \ell_\psi)]
\le \;& 4\mathcal{R}_n^{\mathrm{off}}\left(\mathcal{G}, \frac{1}{4 B_\psi B_M \kappa}\right) \nonumber \\
& + C_1 M^{1-\beta/2} + C_2 M^{2-\beta} \nonumber \\
& + 3\inf_{f_M \in \mathcal{F}} \mathcal{E}(f_M; \ell_\psi),
\end{align}
where $C_1, C_2$ absorb the constants from the tail integration and $\beta$.
Combining this with the rank-based complexity from Theorems \ref{thm:single-head-rank}, \ref{thm:mul-head-rank} and \ref{thm:multi-layer-rank} completes the proof of Theorems \ref{thm:heavy-tailed-erm}, \ref{thm:mh-heavy-tailed-erm} and \ref{thm:ml-heavy-tailed-erm}, respectively.
\end{proof}

%%%%%%%%%%%%%%%%%%%%%%%%%%%%%%%%%%%%%%%%%%%%%%%%%%%%%%%%%%%%%%%%%%%%%%%%%%%%%%%
%%%%%%%%%%%%%%%%%%%%%%%%%%%%%%%%%%%%%%%%%%%%%%%%%%%%%%%%%%%%%%%%%%%%%%%%%%%%%%%

\end{document}